\documentclass{article}

\usepackage{arxiv}
\usepackage[utf8]{inputenc} 
\usepackage[T1]{fontenc}    
\usepackage{hyperref}       
\usepackage{url}            
\usepackage{booktabs}       
\usepackage{amsfonts}       
\usepackage{nicefrac}       
\usepackage{microtype}      
\usepackage{graphicx}
\usepackage{doi}
\usepackage{multirow}
\usepackage{amsmath,amssymb}
\usepackage{amsthm}         
\usepackage{mathtools}      
\usepackage{bm} 
\usepackage{ulem}
\usepackage{algorithmicx}
\usepackage{algpseudocode}
\usepackage{algorithm}
\usepackage{caption}
\usepackage{subcaption}
\usepackage{float}
\usepackage{siunitx}        
\usepackage{xcolor}

\numberwithin{equation}{section}

\definecolor{newcolor}{rgb}{.8,.349,.1}

\theoremstyle{plain} 
\newtheorem{theorem}{Theorem}[section]

\theoremstyle{definition} 

\theoremstyle{remark} 
\newtheorem{remark}[theorem]{Remark}

\title{Muon with Spectral Guidance: Efficient Optimization for Scientific Machine Learning}

\author{
  Binghang Lu\thanks{Equal contribution.} \\
  Elmore Family School of Electrical and Computer Engineering\\
  Purdue University\\
  West Lafayette, IN 47907, USA \\
  \And
  Jiahao Zhang\footnotemark[1] \\
  Department of Mathematics\\
  Purdue University\\
  West Lafayette, IN 47907, USA \\
  \And
  Guang Lin\thanks{Corresponding author: \texttt{guanglin@purdue.edu}} \\
  Department of Mathematics \& School of Mechanical Engineering\\
  Purdue University\\
  West Lafayette, IN 47907, USA \\
}

\begin{document}
\maketitle

\begin{abstract}
Physics-informed neural networks and neural operators often suffer from severe optimization difficulties caused by ill-conditioned gradients, multi-scale spectral behavior, and stiffness induced by physical constraints. Recently, the Muon optimizer has shown promise by performing orthogonalized updates in the singular-vector basis of the gradient, thereby improving geometric conditioning. However, its unit--singular-value updates may lead to overly aggressive steps and lack explicit stability guarantees when applied to physics-informed learning. In this work, we propose \textbf{SpecMuon}, a spectral-aware optimizer that integrates Muon’s orthogonalized geometry with a \textbf{mode-wise relaxed scalar auxiliary variable (RSAV)} mechanism. By decomposing matrix-valued gradients into singular modes and applying RSAV updates individually along dominant spectral directions, SpecMuon adaptively regulates step sizes according to the global loss energy while preserving Muon’s scale-balancing properties. This formulation interprets optimization as a multi-mode gradient flow and enables principled control of stiff spectral components. We establish rigorous theoretical properties of SpecMuon, including a modified energy dissipation law, positivity and boundedness of auxiliary variables, and global convergence with a linear rate under the Polyak--\L{}ojasiewicz condition. Numerical experiments on physics-informed neural networks, DeepONets, and fractional PINN--DeepONets demonstrate that SpecMuon achieves faster convergence and improved stability compared with Adam, AdamW, and the original Muon optimizer on benchmark problems such as the one-dimensional Burgers’ equation and fractional partial differential equations.
\end{abstract}


\section{Introduction}\label{sec:intro}
The rapid development of deep learning has profoundly influenced scientific computing, offering new paradigms for modeling, simulation, and inverse problems governed by partial differential equations (PDEs). In particular, physics-informed neural networks (PINNs)~\cite{PINN1, PINN2,lu2025ipinner} and neural operators, such as DeepONets~\cite{DEEPONET,lu2025fpinn} and Fourier Neural Operators~\cite{FNO, FNO2}, provide flexible frameworks for embedding physical laws into learning architectures or approximating solution operators between infinite-dimensional function spaces. These approaches have demonstrated considerable promise in accelerating simulations, enabling data-efficient learning, and handling high-dimensional problems that are challenging for traditional numerical solvers.

Despite these successes, optimization remains a central bottleneck in physics-informed machine learning~\cite{PINN3, PINN4}. Unlike standard supervised learning tasks in computer vision or natural language processing, the loss landscapes arising from physics constraints are often highly ill-conditioned, characterized by stiff gradients, strong anisotropy, and pronounced multi-scale spectral structure. The simultaneous presence of data-fitting terms, PDE residuals, and boundary or initial conditions can induce severe curvature imbalance across parameter directions, frequently leading to slow convergence, instability, or sensitivity to hyperparameter choices when using standard first-order optimizers such as stochastic gradient descent or Adam \cite{kingma2014adam}. These difficulties become even more pronounced in operator learning settings, where matrix-valued parameters and coupled subnetworks further complicate the optimization geometry~\cite{ENERGY1, ENERGY2, ENERGY3}.

To mitigate these challenges, recent research has explored optimization methods that incorporate geometric or spectral information beyond simple coordinate-wise scaling~\cite{Geo1, Geo2, Geo3}. Among these, the Muon optimizer~\cite{jordan2024muon} has attracted attention for its distinctive update strategy, which operates in the singular-vector basis of gradient matrices rather than the ambient parameter coordinates \cite{jordan2024muon}. By orthogonalizing updates and equalizing singular values, Muon effectively reshapes the optimization geometry and reduces interference between directions with disparate curvature \cite{bernstein2024old}. Although originally developed for large-scale language and vision models, Muon’s geometric perspective is particularly appealing for scientific machine learning, where matrix-valued parameters and spectral imbalance are ubiquitous.

However, the standard Muon update deliberately discards singular value magnitudes, prioritizing geometric displacement over energy control. While this normalization stabilizes the spectrum of the update, it may also produce overly aggressive steps along dominant modes and lacks explicit mechanisms to ensure stability or monotonic decrease of the training loss. In contrast, the scalar auxiliary variable (SAV) methodology, originally developed for gradient flows in computational physics, offers a complementary perspective \cite{shen2018convergence,shen2018scalar,shen2019new}. By augmenting the system with an auxiliary scalar tied to the energy, SAV-based schemes guarantee unconditional stability and modified energy dissipation. Recent adaptations of SAV ideas to neural network optimization~\cite{RVAV1, RVAV2} have demonstrated that energy-based mechanisms can improve robustness and convergence behavior \cite{zhang2024energy,liu2020aegd,liu2023efficient}, yet these approaches typically operate on aggregated gradients and do not explicitly address the spectral structure inherent in physics-informed models.

This observation reveals a fundamental gap in current optimization strategies: geometric conditioning and energy stability are often treated separately, despite both being critical for efficient training in stiff, physics-constrained settings. In this work, we seek to bridge this gap by introducing a spectral-aware optimization framework that unifies orthogonalized geometry with principled energy control. Our central idea is to view optimization as a multi-mode gradient flow, where different spectral components of the gradient evolve under distinct stability constraints.

Motivated by this perspective, we propose \emph{SpecMuon}, a novel optimizer that integrates Muon’s singular-vector-based updates with a mode-wise relaxed scalar auxiliary variable (RSAV) mechanism \cite{jiang2022improving}. By decomposing matrix-valued gradients into their dominant singular directions and applying RSAV updates individually to these modes, SpecMuon adaptively regulates step sizes according to the global loss energy while preserving Muon’s scale-balancing properties. This design enables selective damping of stiff spectral components without sacrificing rapid progress along well-conditioned directions.

The main contributions of this paper are summarized as follows:
\begin{itemize}
\item \textbf{Algorithm Design:} We introduce SpecMuon, a spectral-aware optimizer that combines orthogonalized gradient updates with a mode-wise RSAV mechanism, providing a unified treatment of geometric conditioning and energy stability.
\item \textbf{Theoretical Analysis:} We establish rigorous theoretical guarantees for the proposed method, including a modified energy dissipation law, positivity and boundedness of auxiliary variables, and global convergence with a linear rate under the Polyak--\L{}ojasiewicz condition.
\item \textbf{Numerical Validation:} We demonstrate the effectiveness of SpecMuon through extensive numerical experiments on physics-informed neural networks, DeepONets, and fractional PINN--DeepONets, showing improved convergence speed and stability compared with Adam, AdamW, and the original Muon optimizer.
\end{itemize}

The remainder of this paper is organized as follows. Section~2 reviews the Muon optimizer and the scalar auxiliary variable framework. Section~3 presents the SpecMuon methodology and its theoretical analysis. Numerical experiments are reported in Section~4, and concluding remarks and future directions are given in Section~5.

\section{Background and Related Work}\label{Background}

We consider the unconstrained optimization problem
\[
  \min_{\Theta}\; f(\Theta), 
  \qquad 
  \Theta=\{W^{(\ell)}\in\mathbb{R}^{m_\ell\times n_\ell}\}_{\ell=1}^{L},
\]
where $f$ denotes the neural-network training loss. Although our numerical studies come from scientific computing (e.g., PDE-driven models), the method operates directly on any matrix-valued parameters optimization problem.

\subsection{Muon}
\label{sec:muon}

Let $G^{(\ell)}:=\nabla_{W^{(\ell)}} f(\Theta)\in\mathbb{R}^{m_\ell\times n_\ell}$ be the block gradient. Its thin singular value decomposition (SVD) is
\[
  G^{(\ell)} = U^{(\ell)} \Sigma^{(\ell)} \big(V^{(\ell)}\big)^{\!\top},
\]
and we define the Frobenius-orthonormal rank-one directions
\[
  Q^{(\ell)}_i := u^{(\ell)}_i \big(v^{(\ell)}_i\big)^{\!\top}, 
  \qquad 
  \langle Q^{(\ell)}_i,Q^{(\ell)}_j\rangle_F=\delta_{ij},\;\; \|Q^{(\ell)}_i\|_F=1.
\]
The original Muon step replaces the singular values by ones and moves in the orthogonalized direction spanned by the singular vectors. With stepsize $\eta>0$ and, in practice, a rank budget $k_\ell\le \mathrm{rank}(G^{(\ell)})$, the update reads
\[
  \Delta W^{(\ell)}_{\text{Muon}}
  \;=\;
  -\,\eta \sum_{i=1}^{k_\ell} Q^{(\ell)}_i
  \;=\;
  -\,\eta\, U^{(\ell)} \big(V^{(\ell)}\big)^{\!\top}.
\]
This ``unit-SV'' update preserves the singular vectors and discards the magnitudes, which stabilizes the spectrum of the step and reduces interference between directions with disparate curvature. 

In practice, Muon implements the unit-SV transformation without an explicit SVD. It approximates the polar factor of a normalized update matrix via a few iterations of the Newton–Schulz method, e.g.,
\[
  X_{t+1} \;=\; \tfrac{3}{2}\,X_t \;-\; \tfrac{1}{2}\,X_t X_t^{\!\top} X_t,
\]
which drives the singular values of $X_t$ toward $1$ while preserving the singular vectors. This yields an efficient orthogonalized direction on modern hardware.

\subsection{Scalar auxiliary variable (SAV) and the relaxed variant (RSAV)}
\label{sec:sav-rsav}

Consider the gradient flow for the loss
\[
  \dot{\Theta}(t) \;=\; -\,\nabla f\big(\Theta(t)\big),
\]
and introduce a fixed shift $\kappa>0$ so that $f(\Theta)+\kappa>0$. 
Define the scalar auxiliary variable:
\[
  r(t) \;:=\; \sqrt{\,f\big(\Theta(t)\big)+\kappa\,}.
\]
The SAV idea couples $\Theta$ and $r$ through the extended system
\begin{equation}
\label{eq:cts-sav}
  \dot{\Theta}(t) \;+\; \frac{r(t)}{\sqrt{\,f(\Theta(t))+\kappa\,}}\;\nabla f\big(\Theta(t)\big) \;=\; 0,
  \qquad
  \dot{r}(t) \;=\; \frac{1}{2\sqrt{\,f(\Theta(t))+\kappa\,}}\,
                  \nabla f\big(\Theta(t)\big)^{\!\top}\dot{\Theta}(t).
\end{equation}
The pair \eqref{eq:cts-sav} is equivalent to the original gradient flow when $r(t)=\sqrt{f(\Theta(t))+\kappa}$ holds exactly, yet it exposes a simple Lyapunov structure.

\begin{theorem}[Modified energy law]
Let $(\Theta(t),r(t))$ satisfy \eqref{eq:cts-sav}. Then the modified energy $r(t)^2$ dissipates:
\[
  \frac{d}{dt}\,r(t)^2 \;=\; -\,\frac{r(t)^2}{\,f(\Theta(t))+\kappa\,}\;\big\|\nabla f\big(\Theta(t)\big)\big\|_F^2 \;\le\; 0.
\]
\end{theorem}

\begin{proof}
Multiply the second equation in \eqref{eq:cts-sav} by $2r(t)$ to get
\[
2r\dot r = \frac{r}{\sqrt{f+\kappa}}\nabla f^\top \dot{\Theta}.
\]
Combine it with the first equation, $\dot{\Theta}= -\frac{r}{\sqrt{f+\kappa}}\nabla f$, to obtain:
\[
2r\dot r = -\frac{r^2}{f+\kappa}\|\nabla f\|_F^2,
\]
\[
\frac{d}{dt}r^2\le 0.
\]
\end{proof}

Now let $h>0$ be a time step and denote $G(\Theta):=\nabla f(\Theta)$ and 
$E(\Theta):=\sqrt{f(\Theta)+\kappa}$.
A first-order explicit SAV step is obtained by evaluating $G$ at $\Theta^k$ and by
predicting $r$ with a scalar decoupling:
\begin{align}
  \tilde r^{\,k+1}
  \;&=\;
  \frac{r^{k}}{\,1+\dfrac{h}{2}\,\dfrac{\|G(\Theta^{k})\|_F^2}{\big(E(\Theta^k)\big)^2}\,}, 
  \label{eq:sav-predict-r}\\[2mm]
  \Theta^{k+1}
  \;&=\;
  \Theta^{k}
  \;-\; h\,\frac{\tilde r^{\,k+1}}{E(\Theta^{k})}\,G(\Theta^{k}).
  \label{eq:sav-update-theta}
\end{align}
Formula \eqref{eq:sav-predict-r} comes from discretizing the second equation in
\eqref{eq:cts-sav} and substituting the discrete version of \eqref{eq:sav-update-theta}
to eliminate $\Theta^{k+1}$, which yields a scalar correction for $r$.
The pair \eqref{eq:sav-predict-r}--\eqref{eq:sav-update-theta} decreases the discrete
modified energy $r^2$ for any $h>0$.

In \cite{jiang2022improving}, the SAV predictor $\tilde r^{\,k+1}$ can also be directly connected to the current loss value $E(\Theta^{k+1})$. 
RSAV enforces alignment by a convex relaxation:
\begin{align}
  \tilde r^{\,k+1}
  \;&=\;
  \frac{r^{k}}{\,1+\dfrac{h}{2}\,\dfrac{\|G(\Theta^{k})\|_F^2}{\big(E(\Theta^k)\big)^2}\,},
  \label{eq:rsav-rpred}\\[1mm]
  \Theta^{k+1}
  \;&=\;
  \Theta^{k}
  \;-\; h\,\frac{\tilde r^{\,k+1}}{E(\Theta^{k})}\,G(\Theta^{k}),
  \label{eq:rsav-theta}\\[1mm]
  r^{k+1}
  \;&=\;
  \xi^{k}\,\tilde r^{\,k+1}
  \;+\;
  \bigl(1-\xi^{k}\bigr)\,E(\Theta^{k+1}),
  \qquad \xi^{k}\in[0,1].
  \label{eq:rsav-relax}
\end{align}
The relaxation parameter $\xi^{k}$ is chosen to guarantee a discrete dissipation law.
Let
\[
  {D}^{k+1} \;:=\; \frac{1}{h}\,\bigl\|\Theta^{k+1}-\Theta^{k}\bigr\|_F^{2}.
\]
Fix a tolerance $\psi\in(0,1)$ (e.g., $\psi=0.95$). 
Select $\xi$ as the smallest root in $[0,1]$ of the quadratic inequality
\begin{equation}
\label{eq:rsav-quad}
  \bigl(\xi\,\tilde r^{\,k+1} + (1-\xi)E(\Theta^{k+1})\bigr)^2 
  \;-\; \bigl(\tilde r^{\,k+1}\bigr)^2 \;-\; \bigl(\tilde r^{\,k+1} - r^{\,k}\bigr)^2
  \;\le\; \psi\,{D}^{k+1},
\end{equation}
which has the explicit solution
\[
  \xi^{k} \;=\; 
  \max\!\left\{\,0,\;
  \frac{-\,b - \sqrt{\,b^{2}-4ac\,}}{2a}\right\},
\]
with explicit coefficients $a, b, c$ defined by
\[
  a=\bigl(\tilde r^{\,k+1}-E(\Theta^{k+1})\bigr)^2,
\]
\[
  b=2E(\Theta^{k+1})\bigl(\tilde r^{\,k+1} - E(\Theta^{k+1})\bigr),
\]
\[
  c=\; f(\Theta^{k+1})+\kappa \;-\; \bigl(\tilde r^{\,k+1}\bigr)^2 \;-\; \psi\,{D}^{k+1}.
\]

\begin{theorem}[Discrete energy dissipation for RSAV predictor and RSAV]
Fix $h>0$ and $\kappa>0$. Let
\[
E_k:=\sqrt{f(\Theta^k)+\kappa},\qquad
G_k:=\nabla f(\Theta^k).
\]
Assume the RSAV predictor updates following \eqref{eq:rsav-rpred} and \eqref{eq:rsav-theta}. Define the discrete dissipation term
\[
 D^{k+1}:=\frac{1}{h}\,\big\|\Theta^{k+1}-\Theta^k\big\|_F^2 .
\]
Then the predictor satisfies
\begin{equation}\label{eq:predictor-dissip}
\big(\tilde r^{\,k+1}\big)^2-\big(r^k\big)^2 \;\le\; -\, D^{k+1}\;\le\;0 .
\end{equation}
Moreover, if the relaxation is chosen so that for some $\psi\in[0,1)$ satisfying \eqref{eq:rsav-relax} and \eqref{eq:rsav-quad},
then
\begin{equation}\label{eq:full-dissip}
\big(r^{k+1}\big)^2-\big(r^k\big)^2
\;\le\; -\,(1-\psi)\, D^{k+1}
\;\le\; 0 .
\end{equation}
\end{theorem}

\begin{proof}
From \eqref{eq:rsav-theta},
\[
\Theta^{k+1}-\Theta^k = -\,h\,\frac{\tilde r^{\,k+1}}{E_k}\,G_k,
\quad\Longrightarrow\quad
 D^{k+1}
=\frac{1}{h}\big\| -h\,\tfrac{\tilde r^{\,k+1}}{E_k}\,G_k\big\|_F^2
= h\,\big(\tilde r^{\,k+1}\big)^2\frac{\|G_k\|_F^2}{E_k^2}.
\]
On the other hand, the first predictor formula in \eqref{eq:rsav-rpred} implies
\[
\tilde r^{\,k+1}-r^k
= -\,\frac{h}{2}\,\frac{\|G_k\|_F^2}{E_k^2}\,\tilde r^{\,k+1}.
\]
Multiply both sides by $-2\tilde r^{\,k+1}$ to get the identity
\begin{equation}\label{eq:key}
-\,2\big(\tilde r^{\,k+1}-r^k\big)\,\tilde r^{\,k+1}
= h\,\big(\tilde r^{\,k+1}\big)^2 \frac{\|G_k\|_F^2}{E_k^2}
=  D^{k+1}.
\end{equation}
Note that $\tilde r^{\,k+1}<r^k$ because the denominator in \eqref{eq:rsav-rpred} exceeds $1$.
Hence
\[
\big(\tilde r^{\,k+1}\big)^2-\big(r^k\big)^2
= \big(\tilde r^{\,k+1}-r^k\big)\big(\tilde r^{\,k+1}+r^k\big)
\le \big(\tilde r^{\,k+1}-r^k\big)\,\big(2\tilde r^{\,k+1}\big)
= -\, D^{k+1},
\]
where the last equality uses \eqref{eq:key}. This proves \eqref{eq:predictor-dissip}.

\noindent By assumption \eqref{eq:rsav-quad}, we have
\[
\big(r^{k+1}\big)^2 - \big(\tilde r^{\,k+1}\big)^2
\;\le\; \psi\, D^{k+1} + \big(\tilde r^{\,k+1}-r^k\big)^2 .
\]
Add this inequality to the exact identity
\(
\big(\tilde r^{\,k+1}\big)^2 - \big(r^k\big)^2
= -\, D^{k+1} - \big(\tilde r^{\,k+1}-r^k\big)^2
\)
(from \eqref{eq:key}) to obtain
\[
\big(r^{k+1}\big)^2 - \big(r^k\big)^2
\;\le\; -\,(1-\psi)\, D^{k+1},
\]
which is \eqref{eq:full-dissip}. Because $\psi\in[0,1)$ and $ D^{k+1}\ge 0$, this yields
monotone decay of the modified energy $r^2$.
\end{proof}

\section{Methodology}\label{sec:method}
The core idea of the proposed method is to decompose each matrix gradient into its singular directions and update mode by mode.
Let
\[
  G \;=\; U\Sigma V^{\top} \;=\; \sum_{i} \sigma_i\, u_i v_i^{\top},
  \qquad
  Q_i := u_i v_i^{\top}.
\]
A standard Muon step replaces the singular values by ones and moves in the orthogonalized direction:
\[
  -\,\eta \sum_{i} Q_i .
\]

\noindent Equivalently, since \( G=\sum_{i}\sigma_i Q_i \), we can write
\[
  -\,\eta \sum_{i} Q_i
  \;=\;
  - \sum_{i} \Big(\frac{\eta}{\sigma_i}\Big)\,\sigma_i Q_i,
\]
which is a projected gradient step in the original spectral basis with per-mode learning rate
\[
  h_i \;=\; \frac{\eta}{\sigma_i}.
\]
Thus a single Muon update can be viewed as a coordinated collection of independent micro-steps along the decoupled singular directions.

\subsection{SpecMuon Algorithm}
\label{subsec:muon-rsav}

We now present the proposed update in a form parallel to Section \ref{sec:sav-rsav} and consistent with the notation of Section 2.
Fix a matrix block $W$ and an iteration index $k$. Define
\[
  E^{k} := \sqrt{\,f(\Theta^{k})+\kappa\,}, 
  \qquad 
  G^{k} := \nabla_{W} f(\Theta^{k}).
\]
Following Section 2.1, decompose the block gradient by a (truncated) SVD,
\[
  G^{k} \;=\; U^{k}\Sigma^{k}(V^{k})^{\!\top}
  \;=\; \sum_{i=1}^{k_r} \sigma_{i}^{\,k}\, Q_{i}^{\,k},
  \qquad 
  Q_{i}^{\,k} := u_{i}^{\,k}(v_{i}^{\,k})^{\!\top},
\]
where the rank-one matrices $\{Q_{i}^{\,k}\}$ are Frobenius-orthonormal:
$\langle Q_{i}^{\,k}, Q_{j}^{\,k}\rangle_F=\delta_{ij}$ and $\|Q_{i}^{\,k}\|_F=1$.
This orthonormality allows us to treat each singular direction as a scalar mode at the current iterate.

A Muon step sets singular values to one and updates as $-\eta\sum_{i} Q_{i}^{\,k}$.
A projected-gradient step in the original spectral basis reads $-\sum_{i} h_i^{\,k}\,\sigma_i^{\,k} Q_i^{\,k}$.
Matching coefficients mode by mode gives $h_i^{\,k}\sigma_i^{\,k}=\eta$, i.e., $h_i^{\,k}=\eta/\sigma_i^{\,k}$.
Large singular values therefore receive smaller effective steps, while small singular values receive larger ones, reproducing Muon’s scale balancing in the spectral basis.
We therefore set
\[
  h_{i}^{\,k} := \frac{\eta}{\sigma_{i}^{\,k}},
  \qquad \eta>0,
  \qquad
  g_{i}^{\,k} := \frac{\sigma_{i}^{\,k}}{E^{k}}.
\]

To each singular direction $Q_{i}^{\,k}$ we attach a scalar auxiliary variable $r_{i}^{\,k}$,
initialized by $r_{i}^{\,0}=E^{0}$.
This is the per-mode analogue of the single $r^{k}$ in Section 2.2, now aligned with the singular basis. Before moving the parameter, we form the RSAV predictor in each singular direction.
Specializing the SAV decoupling of Section 2.2 to step $h_i^{\,k}$ and direction $Q_i^{\,k}$ gives
\[
  \tilde r_{i}^{\,k+1}
  \;=\;
  \frac{r_{i}^{\,k}}{\,1+\dfrac{h_{i}^{\,k}}{2}\,(g_{i}^{\,k})^{2}\,}
  \;=\;
  \frac{r_{i}^{\,k}}{\,1+\dfrac{\eta}{2}\,\dfrac{\sigma_{i}^{\,k}}{(E^{k})^{2}}\,}.
\]
Note that there is no matrix norm appears because $\|Q_{i}^{\,k}\|_{F}=1$. The denominator depends only on $\sigma_{i}^{\,k}$ and $E^{k}$.

\noindent Once $\tilde r_{i}^{\,k+1}$ is computed, we advance the parameter along the singular directions:
\[
  W^{k+1}
  \;=\;
  W^{k}
  \;-\;
  \frac{\eta}{E^{k}}\sum_{i=1}^{k_r}\tilde r_{i}^{\,k+1}\,Q_{i}^{\,k}.
\]
Equivalently, in each mode the projected increment satisfies
\[
  \big\langle W^{k+1}-W^{k},\,Q_{i}^{\,k}\big\rangle_F
  \;=\; -\,h_{i}^{\,k}\,\tilde r_{i}^{\,k+1}\,g_{i}^{\,k},
\]
and orthonormality ensures that contributions from different modes add without cross terms.

\noindent After computing $W^{k+1}$, we evaluate the new energy root
$E^{k+1}:=\sqrt{\,f(\Theta^{k+1})+\kappa\,}$ and align each auxiliary variable with the current loss level through the RSAV relaxation:
\[
  r_{i}^{\,k+1}
  \;=\;
  \xi_{i}^{\,k}\,\tilde r_{i}^{\,k+1}
  + \big(1-\xi_{i}^{\,k}\big)\,E^{k+1},
  \qquad \xi_{i}^{\,k}\in[0,1].
\]
The coefficient $\xi_{i}^{\,k}$ is chosen by the same short quadratic rule as in equation \eqref{eq:rsav-quad}, applied per mode with the quantities $(\tilde r_{i}^{\,k+1}, E^{k+1}, h_i^{\,k}, g_i^{\,k})$.
The predictor provides a stable provisional magnitude and the relaxation uses information at the updated iterate to keep $r_{i}^{\,k+1}$ tied to the current loss scale. At iteration $k{+}1$ the SVD is recomputed on the fresh gradient $G^{k+1}$, yielding $\{Q_{i}^{\,k+1},\sigma_{i}^{\,k+1}\}$; the relaxed values $\{r_{i}^{\,k+1}\}$ serve as inputs to the new predictors, and the cycle repeats.

The proposed algorithm is summarized in Algorithm \ref{alg:MuonSAV} for implementation purpose. Note that the theoretical analysis in Sections~3.1--3.3 focuses on the core Muon-RSAV update without momentum, gradient normalization, or heuristic stabilizers. Algorithm \ref{alg:MuonSAV} includes these additional components for practical performance and numerical robustness. As is standard in optimization analysis, the theoretical results describe the underlying update mechanism, while the full algorithm reflects implementation-level enhancements.

\begin{algorithm}[H]
\caption{SpecMuon}
\label{alg:MuonSAV}
\begin{algorithmic}[1]
\Require Learning rate $\eta$, momentum $\mu$, top singular directions $k$, SAV smoothing factor $\xi$, epsilon $\epsilon$
\State Initialize momentum buffer $B_0 \gets 0$
\State Initialize SAV state vector $r_0 \gets \sqrt{\mathcal{L}_0} \cdot \mathbf{1}$
\For{$t = 1, \ldots$}
    \State Compute gradient $G_t \gets \nabla_{\theta} \mathcal{L}_t(\theta_{t-1})$
    \State Normalize gradient $\hat{G}_t \gets G_t / (\|G_t\|_F + \epsilon)$
    \State Compute SVD: $U, S, V^T \gets \mathrm{SVD}(\hat{G}_t)$
    \State Initialize update direction $O_t \gets 0$
    \State \textbf{// SAV update for top-$k$ directions}
    \For{$j = 1, \ldots, k$}
        \State $u_j, s_j, v_j^T \gets j$-th singular components from $U, S, V^T$
        \State Compute SAV learning rate $\eta'_j \gets \eta / (s_j + \epsilon)$
        \State $d_g \gets (s_j \cdot u_j v_j^T) / (\sqrt{\mathcal{L}_t} + \epsilon)$
        \State $r^{\text{new}}_j \gets r_{t-1,j} / (1 + 0.5 \cdot \eta'_j \cdot \|d_g\|_F)$
        \State $O_t \gets O_t + (r^{\text{new}}_j / (\sqrt{\mathcal{L}_t} + \epsilon)) \cdot u_j v_j^T$
        \State \textbf{// Update $r$ for next step}
        \State $T \gets (1-\xi)(r^{\text{new}}_j)^2 + \xi(r_{t-1,j})^2 + (1-\xi)(r^{\text{new}}_j - r_{t-1,j})^2$
        \State $\chi \gets (\sqrt{\mathcal{L}_t} - \sqrt{T}) / (\sqrt{\mathcal{L}_t} - r^{\text{new}}_j + \epsilon)$
        \State $r_{t,j} \gets \mathrm{clamp}(\chi, 0, 1) \cdot r^{\text{new}}_j + (1 - \mathrm{clamp}(\chi, 0, 1)) \cdot \sqrt{\mathcal{L}_t}$
    \EndFor
    \State \textbf{// Muon update for remaining directions}
    \State $U_{k:}, S_{k:}, V^T_{k:} \gets$ remaining singular components
    \State $O_t \gets O_t + U_{k:} \, \mathrm{diag}(S_{k:}) \, V^T_{k:}$
    \State \textbf{// Apply momentum and update parameters}
    \State $B_t \gets \mu B_{t-1} + O_t$
    \State Update parameters $\theta_t \gets \theta_{t-1} - \eta B_t$
\EndFor \\
\Return $\theta_t$
\end{algorithmic}
\end{algorithm}

\subsection{Computational efficiency}
\label{subsec:truncated-r}

The update in Section \ref{subsec:muon-rsav} is defined mode by mode in the singular basis. In practice, we keep only a small number of dominant modes.
We therefore adopt a truncated spectral budget $k_r$ (``top-$k$'' singular directions and implemented as \texttt{rtop} in Algorithem \ref{alg:MuonSAV}) and a matching truncation of auxiliary variables $r_i^{\,k}$. This preserves Muon’s geometric benefits while keeping the additional RSAV overhead negligible.

For a block $W\in\mathbb{R}^{m\times n}$ the dominant costs are:
\begin{itemize}
\item \textbf{Top-$k_r$ SVD / orthogonalization.}
A randomized range finder followed by a small SVD has complexity
$O(mn\,k_r + (m{+}n)k_r^2)$.
When Muon uses a polar (Newton--Schulz) iteration, we still require only a thin orthonormal basis. The cost is linear in $k_r$ once that basis is formed.
\item \textbf{Per-mode RSAV algebra.}
The predictor $\tilde r_i^{\,k+1}$ and the relaxation for each retained mode are scalar operations, for a total of $O(k_r)$ time and memory.
\item \textbf{Aggregated parameter update.}
Forming $\sum_{i=1}^{k_r} \tilde r_i^{\,k+1} Q_i^{\,k}$ is a linear combination and is typically dominated by backprop to obtain $G^{k}$.
\end{itemize}
Hence, relative to Muon, the incremental work of RSAV is $O(k_r)$ scalars plus one evaluation of $E^{k+1}=\sqrt{f(\Theta^{k+1})+\kappa}$. Note that two structural properties justify this truncation. Firstly, in many scientific-computing models the gradient spectrum is skewed, that is, a few singular directions carry most of the energy. Secondly, the RSAV equations are independent per mode. Each update uses only $(\sigma_i^{\,k},E^{k},r_i^{\,k})$ and the orthonormal direction $Q_i^{\,k}$.
Dropping the tail therefore removes small-magnitude contributions without introducing cross-mode errors at the current iterate (orthonormality eliminates cross terms).
If neglected directions later become important, they automatically re-enter when their singular values appear in the top-$k_r$.

Truncating both the spectral basis (to $k_r$ modes) and the auxiliaries (to $\{r_i\}_{i=1}^{k_r}$) yields an update whose cost is essentially that of Muon plus $O(k_r)$ scalar work.
The method captures the dominant geometry, adapts per-mode magnitudes via RSAV with negligible overhead, and scales gracefully with model size.

\subsection{Analysis of the Algorithm}
At each iterate $k$, the gradient of a matrix block admits the orthonormal singular basis
$G^{k}=\sum_{i=1}^{k_r}\sigma_i^{\,k}Q_i^{\,k}$ with $\langle Q_i^{\,k},Q_j^{\,k}\rangle_F=\delta_{ij}$, and our update in Section \ref{subsec:muon-rsav} applies the RSAV predictor and relaxation per mode using the Muon-induced steps $h_i^{\,k}=\eta/\sigma_i^{\,k}$. Because these directions are Frobenius-orthonormal, all inner products that enter the discrete identities (e.g., the projected increments $\langle W^{k+1}-W^{k},Q_i^{\,k}\rangle_F$) decouple across $i$. That is, there are no cross terms at the current iterate. Thus, every structural property introduced in Section 2, Section \ref{subsec:muon-rsav} and Section \ref{subsec:truncated-r} can be proved for a single singular direction and then summed over $i=1,\dots,k_r$ to obtain the block-level statement. Although the SVD (and hence $\{Q_i^{\,k}\}$) is recomputed at each step, our estimates are established within iteration $k$ and extend to the full update by simple summation. We therefore present the analysis in a single mode setting and the extension to the total $G^{k}$ is immediate.

\begin{theorem}[Modified-energy dissipation (mode-wise and global)]
\label{thm:mod-energy}
Suppose the Muon--RSAV update of Section \ref{subsec:muon-rsav} with
per-mode step sizes $h_i^{\,k}=\eta/\sigma_i^{\,k}$ and predictors
$\tilde r_i^{\,k+1}=r_i^{\,k}\big/\!\bigl(1+\tfrac{h_i^{\,k}}{2}(g_i^{\,k})^2\bigr)$,
where $g_i^{\,k}=\sigma_i^{\,k}/E^{k}$ and $E^{k}=\sqrt{f(\Theta^{k})+\kappa}$.
After the parameter update $W^{k+1}$, let $E^{k+1}=\sqrt{f(\Theta^{k+1})+\kappa}$ and define the
per-mode relaxation
\[
r_i^{\,k+1}=\xi_i^{\,k}\,\tilde r_i^{\,k+1}+(1-\xi_i^{\,k})\,E^{k+1},\qquad \xi_i^{\,k}\in[0,1],
\]
with $\xi_i^{\,k}$ chosen by the quadratic rule of \eqref{eq:rsav-quad} using a fixed
$\psi\in(0,1)$.
Then, for each retained singular mode $i$,
\[
\bigl(r_i^{\,k+1}\bigr)^2-\bigl(r_i^{\,k}\bigr)^2
\;\le\; -\,(1-\psi)\,D_i^{\,k+1}\;\le\;0,
\qquad
D_i^{\,k+1}:=\frac{1}{h_i^{\,k}}\big\langle W^{k+1}-W^{k},\,Q_i^{\,k}\big\rangle_F^{2}.
\]
Summing over the $k_r$ retained modes yields the global dissipation of the modified energy
$ M^{k}:=\sum_{i=1}^{k_r}\bigl(r_i^{\,k}\bigr)^2$:
\[
 M^{k+1}- M^{k}
\;\le\; -\,(1-\psi)\sum_{i=1}^{k_r} D_i^{\,k+1}
\;\le\; 0.
\]
\end{theorem}

\begin{proof}
Fix an iteration $k$ and a single singular direction $Q_i^{\,k}$.
The update in one sigular direction in Section \ref{subsec:muon-rsav} is
exactly the RSAV predictor with step $h=h_i^{\,k}$ in Section \ref{sec:sav-rsav}:
\[
\tilde r_i^{\,k+1}
=\frac{r_i^{\,k}}{1+\frac{h_i^{\,k}}{2}(g_i^{\,k})^2},
\qquad
\big\langle W^{k+1}-W^{k},Q_i^{\,k}\big\rangle_F
=-h_i^{\,k}\,\tilde r_i^{\,k+1}\,g_i^{\,k}.
\]
Therefore the standard identity from Section \ref{sec:sav-rsav} holds per mode:
\[
D_i^{\,k+1}=\frac{1}{h_i^{\,k}}\big\langle W^{k+1}-W^{k},Q_i^{\,k}\big\rangle_F^{2}
= -\,2\bigl(\tilde r_i^{\,k+1}-r_i^{\,k}\bigr)\tilde r_i^{\,k+1}.
\]
With the RSAV relaxation
$r_i^{\,k+1}=\xi_i^{\,k}\tilde r_i^{\,k+1}+(1-\xi_i^{\,k})E^{k+1}$ and the same quadratic
choice of $\xi_i^{\,k}$ as in Section \ref{sec:sav-rsav}, the single-mode inequality
\[
\bigl(r_i^{\,k+1}\bigr)^2-\bigl(r_i^{\,k}\bigr)^2
\;\le\; -\,(1-\psi)\,D_i^{\,k+1}
\]
follows verbatim from the proof in Section \ref{sec:sav-rsav} by replacing $h\mapsto h_i^{\,k}$ and
$g\mapsto g_i^{\,k}$.
Finally, the Frobenius-orthonormality $\langle Q_i^{\,k},Q_j^{\,k}\rangle_F=\delta_{ij}$
ensures that $D_i^{\,k+1}$ is additive across $i$ (no cross terms at iteration $k$),
so summing the single-mode inequalities yields the final statement: 
\[
 M^{k+1}- M^{k}
\;\le\; -\,(1-\psi)\sum_{i=1}^{k_r} D_i^{\,k+1}
\;\le\; 0.
\]
\end{proof}

\begin{theorem}[Positivity and uniform lower bound for \(r_i^{\,k}\)]
\label{thm:r-positive}
Assume \(\kappa>0\) and that the objective is bounded below:
\[
f_* \;:=\; \inf_{\Theta} f(\Theta) \;>\; -\infty.
\]
Initialize \(r_i^{\,0}=E^{0}=\sqrt{f(\Theta^{0})+\kappa}\).
Consider the per–mode Muon--RSAV update
\[
\tilde r_i^{\,k+1}
=\frac{r_i^{\,k}}{\,1+\dfrac{h_i^{\,k}}{2}\,(g_i^{\,k})^{2}\,},\qquad
r_i^{\,k+1}=\xi_i^{\,k}\,\tilde r_i^{\,k+1}+(1-\xi_i^{\,k})\,E^{k+1},
\]
with \(g_i^{\,k}=\sigma_i^{\,k}/E^{k}\), \(E^{k}=\sqrt{f(\Theta^{k})+\kappa}\), and \(\xi_i^{\,k}\in[0,1]\).
Then for all \(k\ge 0\),
\[
r_i^{\,k} \;>\; 0
\qquad\text{and}\qquad
\tilde r_i^{\,k+1} \;>\; 0
\]
\end{theorem}

\begin{proof}
By assumption \(f_*>-\infty\), hence \(f(\Theta^{k})+\kappa \ge \kappa+f_* > 0\) for all \(k\), so
\(
E^{k}=\sqrt{f(\Theta^{k})+\kappa}\ge \sqrt{\kappa+f_*}>0.
\)
Induction: \(r_i^{\,0}=E^{0}>0\). If \(r_i^{\,k}>0\), then
\(
\tilde r_i^{\,k+1}=\dfrac{r_i^{\,k}}{1+\frac{h_i^{\,k}}{2}(g_i^{\,k})^{2}} \in (0,r_i^{\,k}],
\)
and the relaxation is a convex combination of positive terms:
\[
r_i^{\,k+1}
=\xi_i^{\,k}\,\tilde r_i^{\,k+1}+(1-\xi_i^{\,k})\,E^{k+1}
\;\ge\; (1-\xi_i^{\,k})\,E^{k+1}
\;\ge\; 0.
\]
Hence \(r_i^{\,k+1}>0\).
\end{proof}

\begin{remark}
The assumption \(f_*>-\infty\), is mild and holds for standard losses (e.g., MSE, cross-entropy), where typically \(f_*=0\).
If only a known bound \(f(\Theta)\ge -B\) is available, choosing \(\kappa>B\) guarantees the same uniform positivity via \(\sqrt{\kappa-B}\).
\end{remark}

Now we move to the dissipation of the original energy for sufficiently small step size (one singular direction). We first list the standing assumptions used in this theorem, then prove a one–mode descent result. The extension to the block level follows by summation across modes, as explained at the beginning of Section 3.3.
\begin{enumerate}
\item[(A1)] \textbf{$L$–smoothness:}
The loss function $f:\mathbb{R}^{m\times n}\to\mathbb{R}$ has $L$–Lipschitz continuous gradient with respect to the Frobenius norm, i.e.
\[
\|\nabla f(X) - \nabla f(Y)\|_F \le L\,\|X - Y\|_F
\quad \text{for all } X,Y.
\]
Equivalently, $f$ satisfies the standard descent inequality
\[
f(Y) \le f(X) + \langle \nabla f(X),\, Y - X \rangle_F + \frac{L}{2}\|Y - X\|_F^{2},
\]
which will be used in the analysis below.

\item[(A2)] \textbf{Lower-bounded loss and shift.} $f_*:=\inf_{\Theta} f(\Theta)>-\infty$ and $\kappa>0$, so that
$E^{k}=\sqrt{f(\Theta^{k})+\kappa}\ge \sqrt{\kappa+f_*}>0$.
\end{enumerate}

\begin{theorem}[One–mode dissipation of the original energy for a sufficiently small step]
\label{thm:orig-descent-one-mode}
Under assumptions (A1) and (A2), the one–step decrease
\(
f(\Theta^{k+1})\le f(\Theta^{k})
\)
holds in the singular direction $Q_i^{\,k}$ provided the stepsize parameter $\eta$ satisfies
\eqref{eq:eta-local-condition}, i.e.
\[
\eta \;\le\; \frac{2\,\sigma_i^{\,k}\,E^{k}}{L\,\tilde r_i^{\,k+1}}.
\]
\end{theorem}

\begin{proof}
Fix an iteration $k$ and a retained singular direction $Q_i^{\,k}$ with
$\langle Q_i^{\,k},Q_i^{\,k}\rangle_F=1$ and coefficient $\sigma_i^{\,k}=\langle G^{k},Q_i^{\,k}\rangle_F$, where $G^{k}=\nabla_W f(\Theta^{k})$.
With $h_i^{\,k}=\eta/\sigma_i^{\,k}$ and $g_i^{\,k}=\sigma_i^{\,k}/E^{k}$, define
\[
\tilde r_i^{\,k+1}=\frac{r_i^{\,k}}{1+\frac{h_i^{\,k}}{2}(g_i^{\,k})^2},
\qquad
W^{k+1}=W^{k}-\frac{\eta}{E^{k}}\tilde r_i^{\,k+1}Q_i^{\,k}.
\]
As in Section 3.1, the relaxation that produces $r_i^{\,k+1}$ is performed after updating $W$ and is not needed here.

\noindent By (A1) with $Y:=W^{k+1}-W^{k}=-(\eta/E^{k})\tilde r_i^{\,k+1}Q_i^{\,k}$ and $\|Q_i^{\,k}\|_F=1$,
\begin{align}
f(\Theta^{k+1})
&\le f(\Theta^{k}) + \Big\langle G^{k},\; -\frac{\eta}{E^{k}}\tilde r_i^{\,k+1}Q_i^{\,k}\Big\rangle_F
\;+\; \frac{L}{2}\,\Big\|\frac{\eta}{E^{k}}\tilde r_i^{\,k+1}Q_i^{\,k}\Big\|_F^{2} \notag\\
&= f(\Theta^{k}) - \frac{\eta}{E^{k}}\tilde r_i^{\,k+1}\,\sigma_i^{\,k}
\;+\; \frac{L}{2}\,\frac{\eta^{2}}{(E^{k})^{2}}\,(\tilde r_i^{\,k+1})^{2}.
\label{eq:one-mode-descent-template}
\end{align}
Thus, a sufficient local condition for $f(\Theta^{k+1})\le f(\Theta^{k})$ in this mode is
\begin{equation}\label{eq:eta-local-condition}
\frac{\eta}{E^{k}}\tilde r_i^{\,k+1} \;\le\; \frac{2\,\sigma_i^{\,k}}{L}
\quad\Longleftrightarrow\quad
\;\eta \;\le\; \frac{2\,\sigma_i^{\,k}\,E^{k}}{L\,\tilde r_i^{\,k+1}}\; .
\end{equation}
\end{proof}

Note that the local condition \eqref{eq:eta-local-condition} is verifiable at runtime because it uses current quantities $(\sigma_i^{\,k},E^{k},\tilde r_i^{\,k+1})$.
For a uniform (iterate–independent) choice of $\eta$, one may use the truncation of the singular values. If we retain only modes with $\sigma_i^{\,k}\ge \sigma_{\min}>0$ and enforce a predictor cap $\tilde r_i^{\,k+1}\le C_r\,E^{k}$ with some $C_r\ge 1$ (e.g., by clipping the ratio $\tilde r_i^{\,k+1}/E^{k}$), then a single, iterate–independent stepsize
\begin{equation}\label{eq:global-eta}
\eta \;\le\; \frac{2\,\sigma_{\min}}{L\,C_r}
\end{equation}
guarantees $f(\Theta^{k+1})\le f(\Theta^{k})$ in every retained mode and hence in total.

\begin{remark}
Summing \eqref{eq:one-mode-descent-template} across the $k_r$ modes, Frobenius–orthonormal directions yields
\(
f(\Theta^{k+1}) \le f(\Theta^{k}) - \frac{\eta}{E^{k}}\sum_{i}\tilde r_i^{\,k+1}\sigma_i^{\,k}
+ \frac{L}{2}\frac{\eta^{2}}{(E^{k})^{2}}\sum_{i}(\tilde r_i^{\,k+1})^{2}.
\)
Thus, if \eqref{eq:eta-local-condition} holds for all retained $i$, the total energy decreases.
This is the precise sense in which proving the one–mode inequality suffices at iteration $k$.
\end{remark}

We now analyze the convergence behavior of the proposed Muon--RSAV algorithm under the same conditions (A1)--(A2) introduced previously, together with an additional regularity assumption stated below.

\begin{enumerate}
\item[(A3)] (\textbf{the Polyak--Łojasiewicz (PL) condition}:) There exists $\mu>0$ such that
\[
\tfrac{1}{2}\|\nabla f(\Theta)\|_F^2 \ge \mu \big(f(\Theta)-f_*\big), \quad \forall\, \Theta.
\]
where $f_*=\inf_{\Theta} f(\Theta)>-\infty$. 
This condition generalizes strong convexity and is commonly satisfied by many overparameterized neural networks and PDE-constrained learning problems.
\end{enumerate}

\begin{theorem}[Convergence and linear rate]
\label{thm:convergence}
Let assumptions (A1)--(A3) hold. 
Then the sequence $\{f(\Theta^k)\}$ generated by the proposed Muon--RSAV method is nonincreasing and convergent. 
Moreover,
\[
\lim_{k\to\infty}\|\nabla f(\Theta^k)\|_F = 0.
\]
If, in addition, the PL condition (A3) holds, 
the method converges linearly to the optimal value $f_*$, i.e.,
\[
f(\Theta^{k+1})-f_* \;\le\; (1-\rho)\,\big(f(\Theta^k)-f_*\big),
\quad
\rho \;=\; \frac{2\tau(2-\tau)}{L}\,\underline E\,c_0\,\mu \in (0,1],
\]
where $\tau\in(0,2]$ is the step–scaling parameter, 
$\underline E=\sqrt{f_*+\kappa}>0$, 
and $c_0=\cos^2\!\angle(\tilde r^{\,k+1},\sigma^{k})\in(0,1]$ 
quantifies the alignment between $\tilde r^{\,k+1}$ and the gradient singular spectrum.
\end{theorem}

\begin{proof}
The one–mode descent inequality has been established in \eqref{eq:one-mode-descent-template}. 
Summing it over all retained spectral directions yields
\begin{equation}\label{eq:descent-block}
f(\Theta^{k+1}) \le f(\Theta^k)
-\frac{\eta}{E^k}\sum_i \tilde r_i^{k+1}\sigma_i^k
+\frac{L}{2}\Big(\frac{\eta}{E^k}\Big)^2\sum_i(\tilde r_i^{k+1})^2.
\end{equation}
Define $S_1^k=\sum_i\tilde r_i^{k+1}\sigma_i^k$ and 
$S_2^k=\sum_i(\tilde r_i^{k+1})^2$. 
For fixed $\Theta^k$, the right-hand side of \eqref{eq:descent-block} is a convex quadratic function of~$\eta$. 
Its minimizer is 
\[
\eta_k^* = \frac{E^k}{L}\,\frac{S_1^k}{S_2^k}.
\]
Setting $\eta_k=\tau\,\eta_k^*$ with $\tau\in(0,2]$ yields
\begin{equation}\label{eq:decrease}
f(\Theta^{k+1})-f(\Theta^k)
\;\le\;
-\frac{\tau(2-\tau)}{2}\,\frac{E^k}{L}\,\frac{(S_1^k)^2}{S_2^k}.
\end{equation}
Since $E^k\ge\underline E>0$ by (A2), the right-hand side of \eqref{eq:decrease} is nonpositive, so $\{f(\Theta^k)\}$ is monotonically decreasing and bounded below by $f_*$. Therefore $f(\Theta^k)$ converges.

\noindent Next, by the Cauchy--Schwarz inequality,
\[
S_1^k = \sum_i \tilde r_i^{k+1}\sigma_i^k 
\;\le\; \sqrt{S_2^k}\, \|\nabla_W f(\Theta^k)\|_F.
\]
Summing \eqref{eq:decrease} over $k$ gives
\[
\sum_{k=0}^\infty \|\nabla f(\Theta^k)\|_F^2 < \infty,
\]
which implies $\lim_{k\to\infty}\|\nabla f(\Theta^k)\|_F=0$. 
Hence, every limit point of $\{\Theta^k\}$ is a stationary point.

\noindent Finally, if (A3) holds, we apply the PL inequality:
\[
\frac{(S_1^k)^2}{S_2^k}
=\Big\langle \frac{\tilde r^{\,k+1}}{\|\tilde r^{\,k+1}\|_2},\,\sigma^k\Big\rangle^2
\ge c_0\,\|\sigma^k\|_2^2
=c_0\,\|\nabla_W f(\Theta^k)\|_F^2.
\]
Substituting into \eqref{eq:decrease} gives
\[
f(\Theta^{k+1})-f(\Theta^k)
\;\le\;
-\frac{2\tau(2-\tau)}{L}\,\underline E\,c_0\,\mu\,
\big(f(\Theta^k)-f_*\big),
\]
which leads to
\[
f(\Theta^{k+1})-f_*
\;\le\;
(1-\rho)\,(f(\Theta^k)-f_*),
\]
where $\rho=\tfrac{2\tau(2-\tau)}{L}\,\underline E\,c_0\,\mu\in(0,1]$.
Therefore, the sequence $\{f(\Theta^k)\}$ converges to $f_*$ at a linear rate.
\end{proof}

\begin{remark}
The factor $c_0=\cos^2\!\angle(\tilde r^{\,k+1},\sigma^{k})$ measures 
the correlation between the auxiliary predictors and the gradient spectrum. 
Empirically, both vectors are nonnegative and dominated by the top singular modes, 
so $c_0$ remains positive in practice. 
Even if $c_0$ becomes small, \eqref{eq:decrease} still guarantees monotone convergence, 
while the PL condition ensures asymptotic linear convergence when $c_0>0$.
\end{remark}

\section{Numerical Results}

\subsection{Toy example: linear regression}
We begin with a linear regression problem to illustrate the behavior of the proposed SpecMuon optimizer in a controlled convex setting. Although linear regression is well understood, it provides a transparent testbed for examining optimization dynamics and spectral effects without the additional complexities introduced by nonlinear architectures.

Specifically, we consider the least-squares problem
\[
\min_{W \in \mathbb{R}^{m \times n}} 
f(W) := \frac{1}{2}\|WX - Y\|_F^2,
\]
where $X \in \mathbb{R}^{n \times N}$ denotes the input data matrix and $Y \in \mathbb{R}^{m \times N}$ the corresponding targets. The gradient takes the closed form
\[
\nabla f(W) = (WX - Y)X^\top,
\]
which is matrix-valued and whose singular spectrum reflects correlations present in the data. Consequently, the problem exhibits anisotropic curvature and serves as a minimal example in which spectral imbalance can influence optimization behavior.

We compare standard gradient-based optimizers, including Adam, AdamW, and the original Muon, with the proposed SpecMuon method. All methods are initialized identically and tuned to their best-performing learning rates. For SpecMuon, we retain only the dominant two singular directions of the gradient and apply the mode-wise RSAV mechanism described in Section~3.

Figure~\ref{fig:linear_regression} reports the training loss trajectories. While all methods eventually converge to the global minimizer, clear differences emerge in the transient regime. Standard first-order methods exhibit slow initial progress due to curvature imbalance. SpecMuon improves convergence speed by adaptively regulating the magnitude of each spectral update through its energy-based auxiliary variables.

This toy example highlights two key aspects of the proposed method. First, even in a convex setting, spectral geometry plays a decisive role in optimization efficiency. Second, incorporating an energy-aware mechanism into spectral updates provides additional robustness without sacrificing convergence speed. These observations motivate the application of SpecMuon to more challenging physics-informed learning problems, which we investigate next.

\begin{figure}[h]
    \centering
    \includegraphics[width=0.5\textwidth]{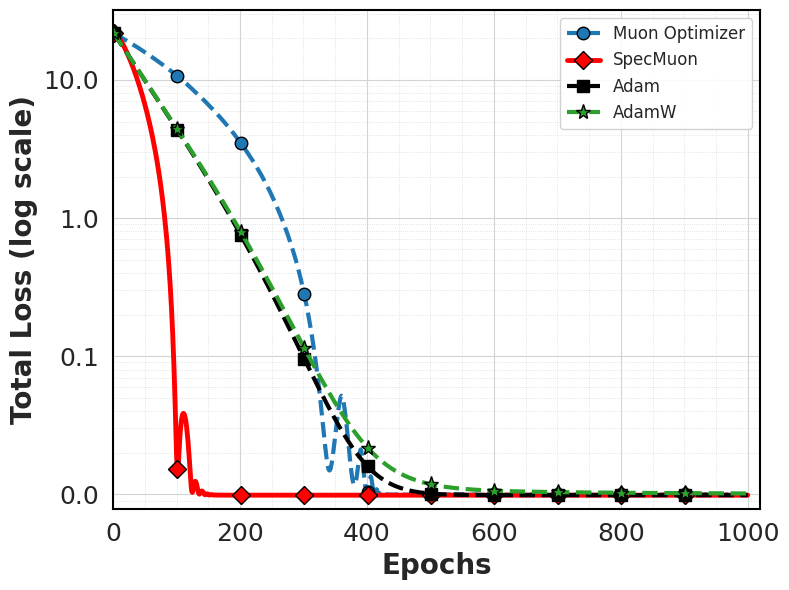}
    \caption{Training loss (log scale) for the linear regression problem. SpecMuon (red) is compared with Muon, Adam, and AdamW. SpecMuon achieves faster convergence by combining spectral orthogonalization with mode-wise energy regulation.}
    \label{fig:linear_regression}
\end{figure}

\subsection{Physics-Informed Neural Network}
We next consider the one-dimensional viscous Burgers’ equation, a nonlinear partial differential equation frequently used as a benchmark problem in scientific machine learning. The equation with Dirichlet boundary conditions is defined on the spatial domain $\Omega = [-1,1]$ and temporal domain $[0,T]$ as follows:
\begin{align}
&\frac{\partial u}{\partial t} - \nu \frac{\partial^2 u}{\partial x^2} + u \frac{\partial u}{\partial x} = 0,\quad x \in \Omega,\; t \in [0,T],\label{eq:burgers}\\[6pt]
&u(x,t) = 0,\quad \forall x \in \partial\Omega,\\[6pt]
&u(x,0) = -\sin(\pi x).
\end{align}
Here, $u(x,t)$ represents the solution over space and time, and $\nu$ is the viscosity chosen to be $0.01/\pi$. 

To ensure a fair evaluation of optimizer performance, we conducted a systematic hyperparameter tuning procedure for all compared methods: Adam, AdamW, Muon, and the proposed SpecMuon. As summarized in Table \ref{tab:hyperparameters_burger}, a grid search was performed over the learning rates to determine the optimal configuration for each optimizer, while other parameters were either set to established default values or tuned within a limited range.

For the baseline optimizers, both Adam and AdamW achieved their best performance with a learning rate of $5 \times 10^{-3}$. The Muon optimizer performed best with a learning rate of $0.02$ and momentum $0.02$. For the proposed SpecMuon method, the grid search identified a learning rate of $3\times 10^{-3}$ and momentum $0.9$ as optimal. In addition, SpecMuon introduces two hyperparameters related to spectral aggregation: $r_{top} = 6$ and $\eta_{sav} = 0.2$.

 \begin{table}[ht] \centering \caption{Hyperparameter configurations for the PINN experiment on Burgers’ equation. Learning rates were selected via grid search. Other parameters were set to standard defaults unless otherwise specified.} \label{tab:hyperparameters_burger} \begin{tabular}{@{}llcc@{}} \toprule \textbf{Optimizer} & \textbf{Hyperparameter} & \textbf{Values Searched} & \textbf{Best Value Used} \\ \midrule 

 \multirow{3}{*}{\texttt{Adam}} & Learning Rate (\text{lr}) & \{0.01, 5e-3, 1e-3, 5e-4\} & 5e-3 \\ & Betas ($\beta_1, \beta_2$) & -- & (0.9, 0.999) \\ & Weight Decay ($\lambda$) & -- & 0 \\ \midrule %
 \multirow{3}{*}{\texttt{AdamW}} & Learning Rate (\text{lr}) & \{0.01, 5e-3, 1e-3, 5e-4\} & 5e-3 \\ & Betas ($\beta_1, \beta_2$) & -- & (0.9, 0.999) \\ & Weight Decay ($\lambda$) & \{0.0, 0.01, 1e-3, 5e-4, 1e-4 \} & 5e-4 \\ \midrule %
\multirow{2}{*}{\texttt{Muon}} & Learning Rate (\text{lr}) & \{0.1, 0.05, 0.02, 5e-3\} & 0.02\\ & Momentum ($\beta$) & \{0.0, 0.02,  1e-3, 1e-4, 5e-4\} & 0.02 \\ \midrule %
{\texttt{SpecMuon}} & Learning Rate (\text{lr}) & \{0.1, 0.05, 0.01, 3e-3\} & 3e-3 \\ & Momentum ($\beta$) & \{0.9,  0.01, 0.02, 1e-3, 1e-4, 5e-4\} & 0.9  \\ & r\textsubscript{top} & \{2 - 8\} & 6 \\ & SAV Eta ($\eta_{sav}$) & \{0.2 - 0.8\} & 0.2 \\ \bottomrule
 \end{tabular} 
 \end{table} 

\begin{figure}[h]
    \centering
    \begin{subfigure}[t]{0.48\textwidth}
        \centering
        \includegraphics[width=\textwidth]{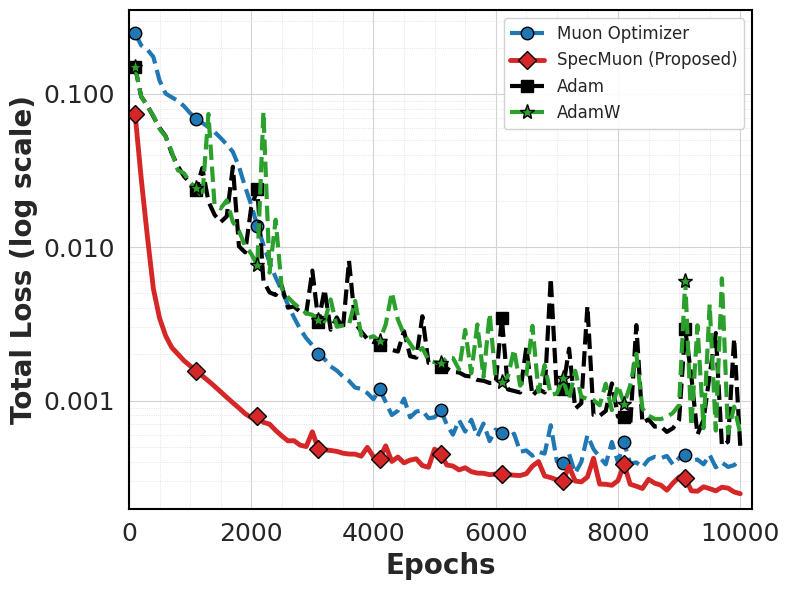}
        \caption{Training loss (log scale) comparison among Adam, AdamW, Muon, and SpecMuon.}
        \label{fig:muon_sav_PINN_burger}
    \end{subfigure}
    \hfill
    \begin{subfigure}[t]{0.48\textwidth}
        \centering
        \includegraphics[width=\textwidth]{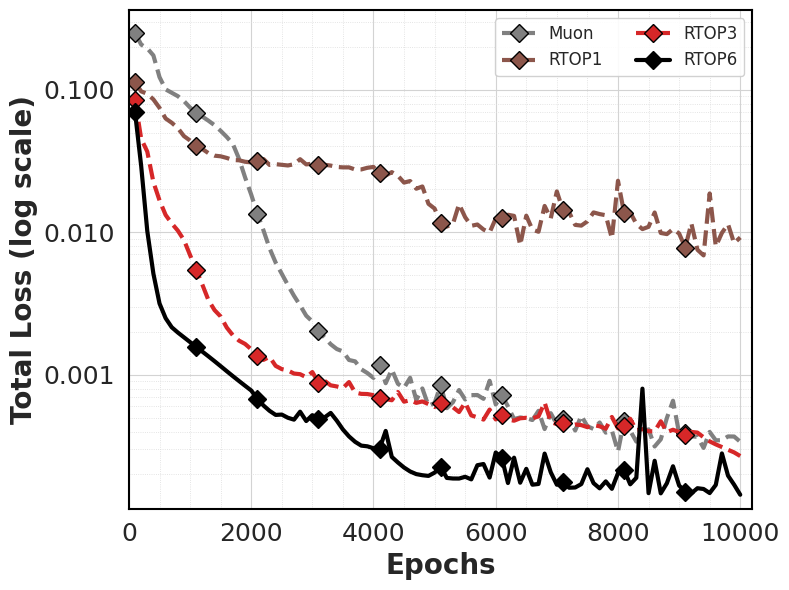}
        \caption{Ablation study on the spectral rank parameter $r_{top}$ in SpecMuon. Retaining multiple dominant singular directions improves convergence stability and speed.}
        \label{fig:muon_sav_rtop_burger}
    \end{subfigure}
    \caption{PINN training on the one-dimensional Burgers’ equation.}
    \label{fig:muon_sav_combined_burger}
\end{figure}

Figure~\ref{fig:muon_sav_PINN_burger} presents the training loss trajectories for the PINN applied to Burgers’ equation. The proposed SpecMuon optimizer (solid red curve) is compared with Adam, AdamW, and the original Muon optimizer. As shown in the figure, SpecMuon demonstrates improved convergence behavior relative to all baseline methods. In particular, it achieves faster initial descent and consistently maintains a lower loss throughout the full training duration of $10,000$ epochs. Figure~\ref{fig:muon_sav_rtop_burger} reports an ablation study on the hyperparmeter $r_{top}$, which contros the number of dominant singular values retained in the spectral aggregation step. The results indicate that incorporating multiple leading spectral directions improves convergence performance, with $r_{top} = 6$ providing a favorable balance between stability and efficiency.

The quantitative results are summarized in Table~\ref{tab:results_burger}. The proposed SpecMuon optimizer achieves the lowest final training loss and the smallest mean squared error (MSE) among all compared methods. Although SpecMuon incurs slightly higher computational cost due to spectral decomposition and auxiliary-variable updates, it delivers improved solution accuracy and more stable convergence behavior.

\begin{table}[ht] \centering \caption{Quantitative comparison for the Burgers’ PINN experiment. Reported metrics include final training loss, mean squared error (MSE), and total training time on an NVIDIA H100 80GB GPU. Best results are shown in bold.} 
\label{tab:results_burger} 
\sisetup{separate-uncertainty, detect-weight=true, detect-inline-weight=math}  
\begin{tabular}{@{}lccc@{}} 
\toprule 
\textbf{Optimizer} & \textbf{Final Training Loss} & \textbf{MSE} & \textbf{Training Time (s)} \\ 
\midrule 
\texttt{Adam} & \num{5.88(12)e-4} & \num{1.42e-3} & \num{120(12)} \\ 
\texttt{AdamW} & \num{7.46(15)e-4} & \num{1.50e-3} & \num{117(11)} \\ 
\texttt{Muon} & \num{3.39(5)e-4} & \num{1.25e-3} & \num{149(15)} \\ 
\textbf{\texttt{SpecMuon}} & \textbf{\num{1.55(4)e-4}} & \textbf{\num{5.06e-4}} & \textbf{\num{180(18)}} \\ 
\bottomrule 
\end{tabular} 
\end{table}

\begin{figure}[H]
    \centering
    \includegraphics[width=.32\linewidth]{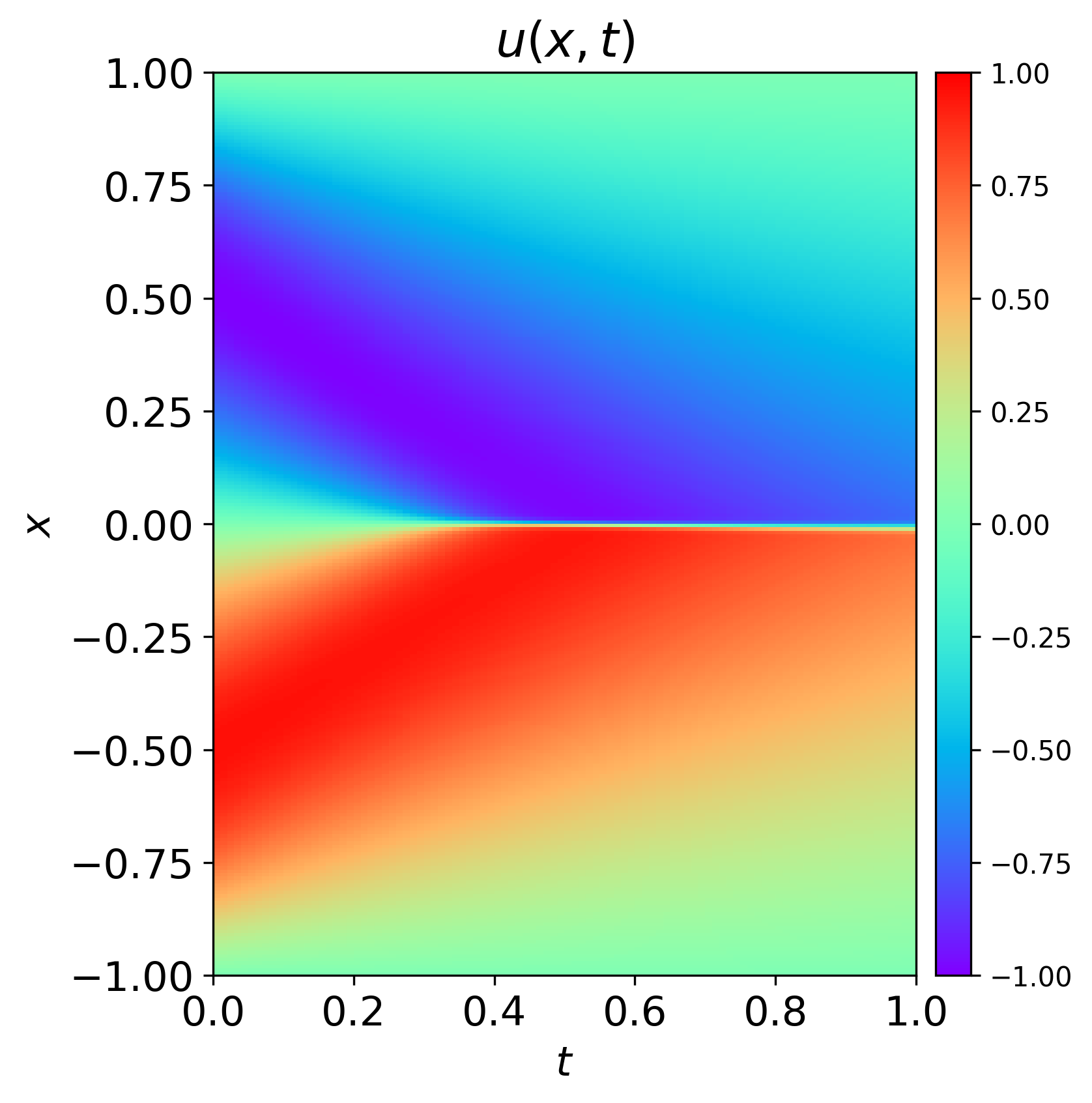}
    \includegraphics[width=.32\linewidth]{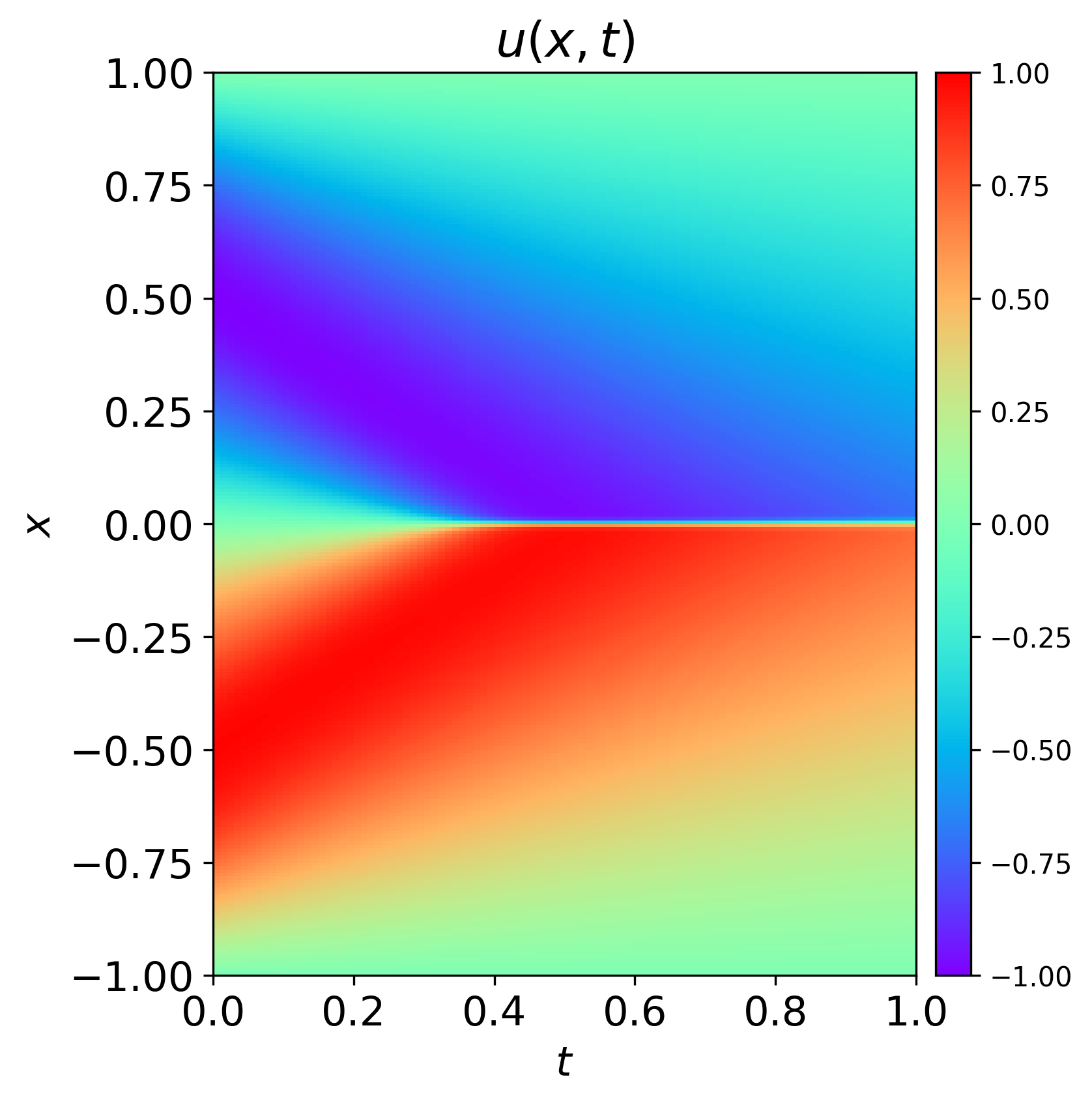}
    \includegraphics[width=.32\linewidth]{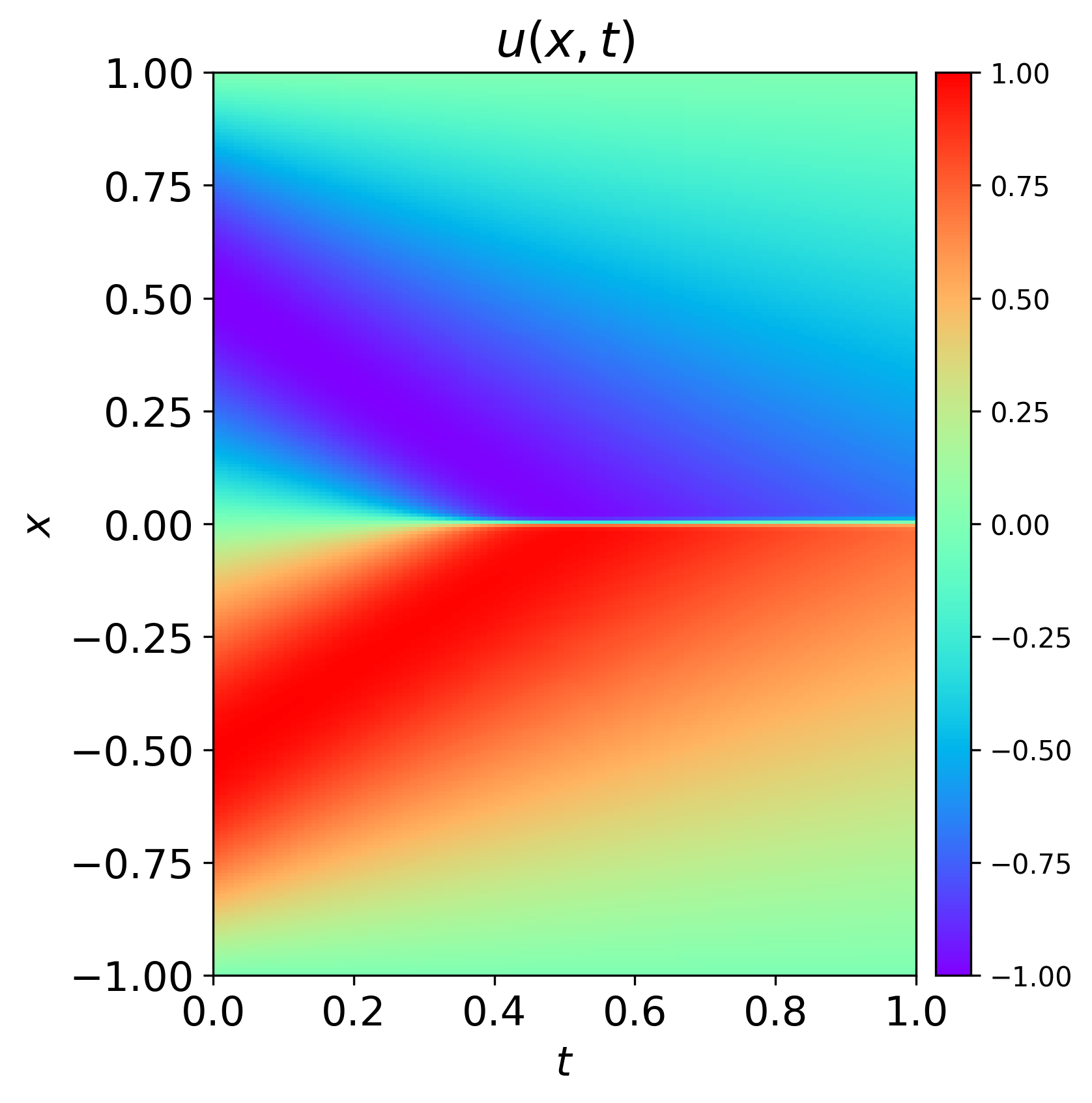}
    \caption{Spatiotemporal solution $u(x, t)$ of the Burgers’ equation obtained by PINNs trained with Adam(left), Muon(middle), and SpecMuon(right). SpecMuon produces a more accurate and smoother approximation of the reference solution.}
    \label{fig:nsga-kf-20}
\end{figure}

\begin{figure}[H]
    \centering
    \includegraphics[width=.32\linewidth, height=0.30\textheight]{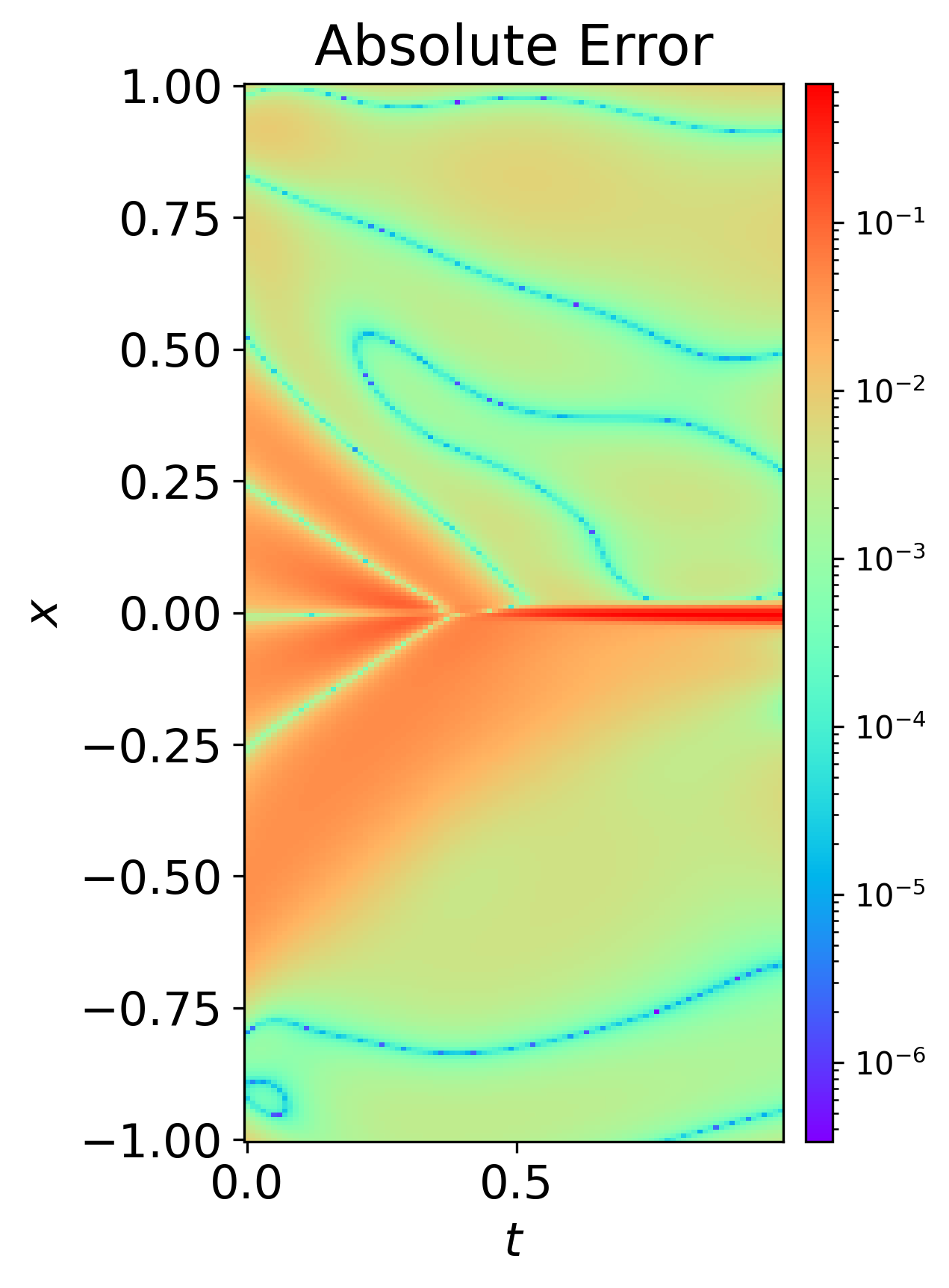}
    \includegraphics[width=.32\linewidth,height=0.30\textheight]{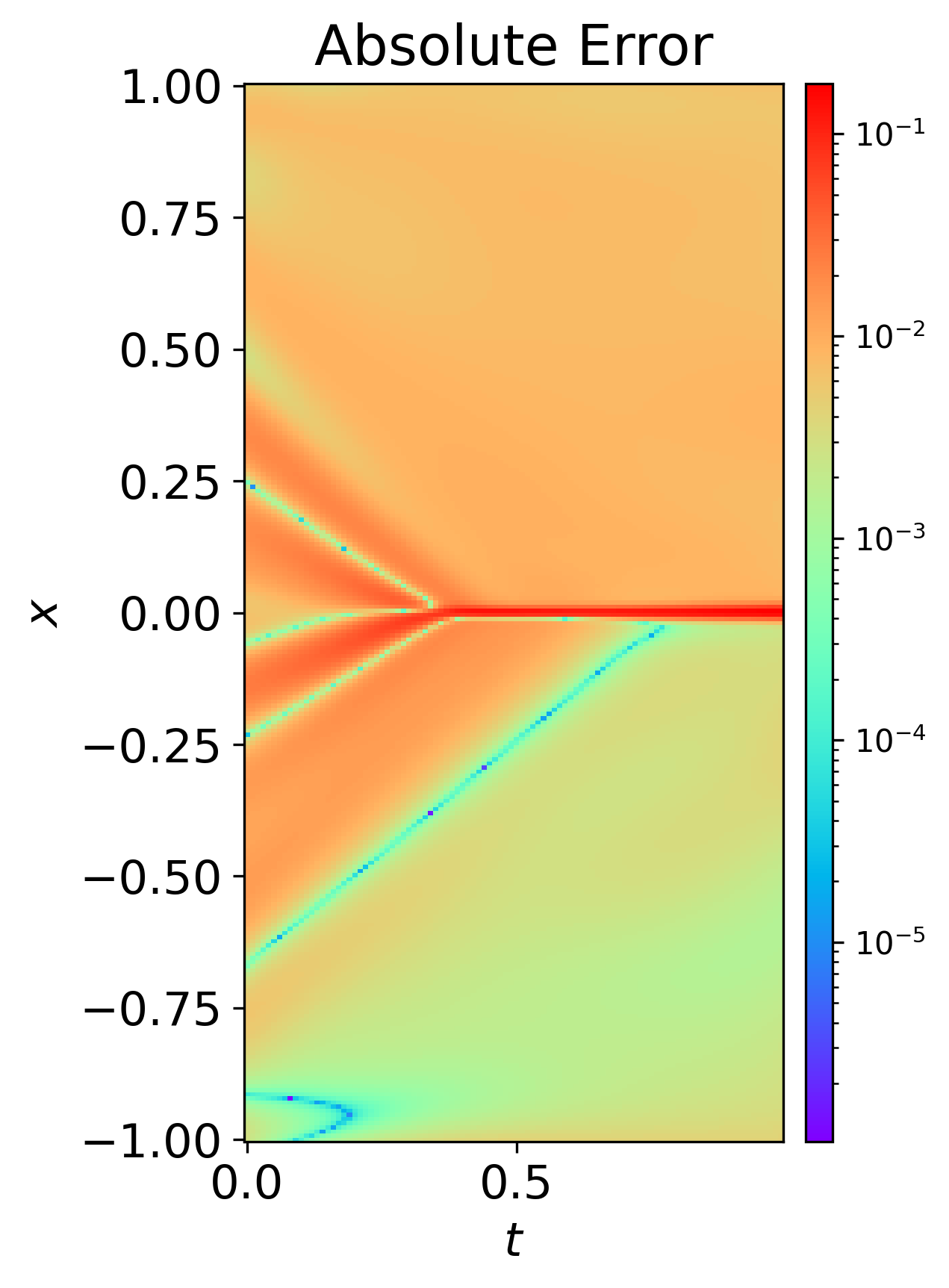}
    \includegraphics[width=.32\linewidth,height=0.30\textheight]{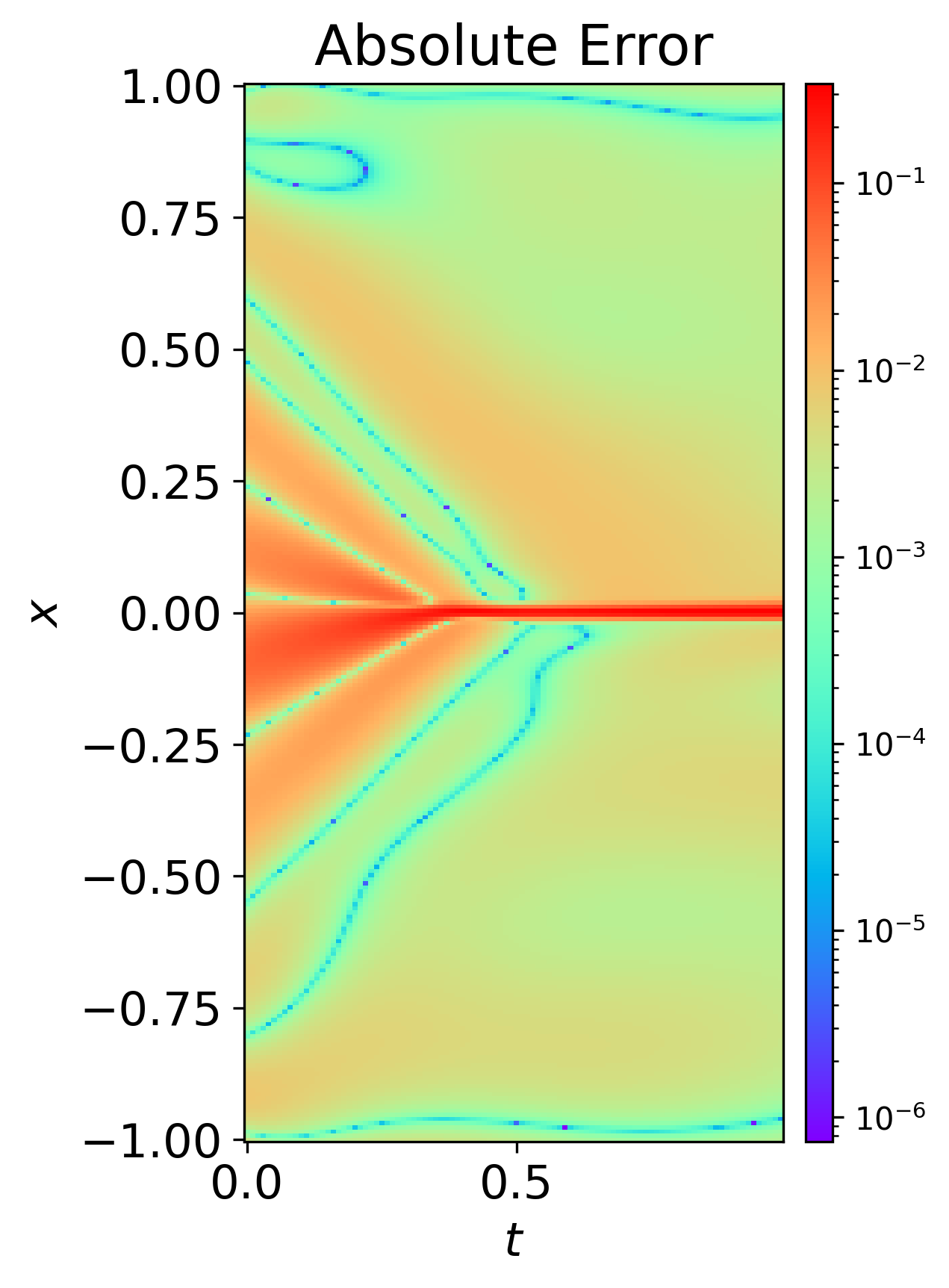}
    \caption{Spatial distribution of the $L_1$ error for the Burgers’ PINN trained with Adam(left), Muon(middle), and SpecMuon(right). SpecMuon exhibits reduced error magnitude.}
    \label{fig:nsga-kf-20}
\end{figure}

\subsection{DeepONet}

To further evaluate the performance of the proposed optimizer in an operator-learning setting, we apply it to a DeepONet model trained on Burgers’ equation~\ref{eq:burgers}. This experiment assesses whether the spectral–energy coupling mechanism remains effective in architectures with coupled branch and trunk networks.

The hyperparameter configurations obtained through grid search are summarized in Table~\ref{tab:hyperparameters_buger2}. For Adam and AdamW, the optimal learning rate was $5 \times 10^{-3}$. For the proposed SpecMuon method, the grid search identified an optimal learning rate of $1 \times 10^{-3}$ and momentum $0.9$. Notebly, the optimal spectral rank parameter was slightly higher for the DeepONet architecture, with $r_{\text{top}} = 8$, reflecting the richer gradient structure induced by operator learning.

\begin{table}[ht] \centering \caption{Hyperparameter configurations for the DeepONet experiment on Burgers’ equation. Learning rates were selected via grid search. The optimal spectral rank $r_{top}$ increases compared with the PINN setting due to richer gradient structure.} \label{tab:hyperparameters_buger2} \begin{tabular}{@{}llcc@{}} \toprule \textbf{Optimizer} & \textbf{Hyperparameter} & \textbf{Values Searched} & \textbf{Best Value Used} \\ \midrule 

 \multirow{3}{*}{\texttt{Adam}} & Learning Rate (\text{lr}) & \{0.01, 5e-3, 1e-3, 5e-4\} & 5e-3 \\ & Betas ($\beta_1, \beta_2$) & -- & (0.9, 0.999) \\ & Weight Decay ($\lambda$) & -- & 0 \\ \midrule %
 \multirow{3}{*}{\texttt{AdamW}} & Learning Rate (\text{lr}) & \{0.01, 5e-3, 1e-3, 5e-4\} & 5e-3\\ & Betas ($\beta_1, \beta_2$) & -- & (0.9, 0.999) \\ & Weight Decay ($\lambda$) & \{0.0, 0.01, 1e-3, 5e-4, 1e-4\} & 1e-2 \\ \midrule %
\multirow{2}{*}{\texttt{Muon}} & Learning Rate (\text{lr}) & \{0.1, 0.05, 5e-3, 1e-3\} & 1e-3 \\ & Momentum ($\beta$) & \{0.0, 0.02,  1e-3, 1e-4, 5e-4\} & 1e-4 \\ \midrule %
{\texttt{SpecMuon}} & Learning Rate (\text{lr}) & \{0.1, 0.05, 0.01, 1e-3, 1e-4\} & 1e-3 \\ & Momentum ($\beta$) & \{0.9,  0.01, 0.02, 1e-3, 1e-4, 5e-4\} & 0.9  \\ & r\textsubscript{top} & \{2 - 8\} & 8 \\ & SAV Eta ($\eta_{sav}$) & \{0.2 - 0.8\} & 0.2 \\ \bottomrule
 \end{tabular} 
 \end{table} 

Figure~\ref{fig:muon_sav_PINN_burger2} presents the training loss trajectories for the DeepONet model applied to Burgers’ equation. The proposed SpecMuon optimizer (solid red curve) is compared with Adam, AdamW, and the original Muon optimizer. As illustrated in the figure, SpecMuon exhibits faster convergence and achieves lower training loss throughout the training process. Figure~\ref{fig:muon_sav_rtop_burger2} shows the ablation study with respect to the spectral rank parameter $r_{top}$. The results indicate that retaining a moderate number of dominant singular directions improves convergence behavior, with $r_{\text{top}} = 8$ yielding the best empirical performance in this setting.

\begin{figure}[h]
    \centering
    \begin{subfigure}[t]{0.48\textwidth}
        \centering
        \includegraphics[width=\textwidth]{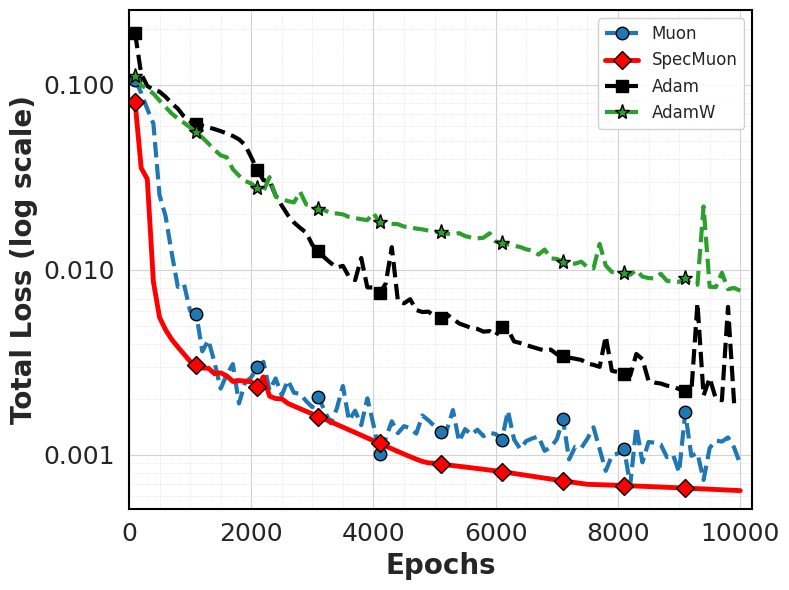}
        \caption{Training loss(log scale) comparison among Adam, AdamW, Muon, and SpecMuon.}
        \label{fig:muon_sav_PINN_burger2}
    \end{subfigure}
    \hfill
    \begin{subfigure}[t]{0.48\textwidth}
        \centering
        \includegraphics[width=\textwidth]{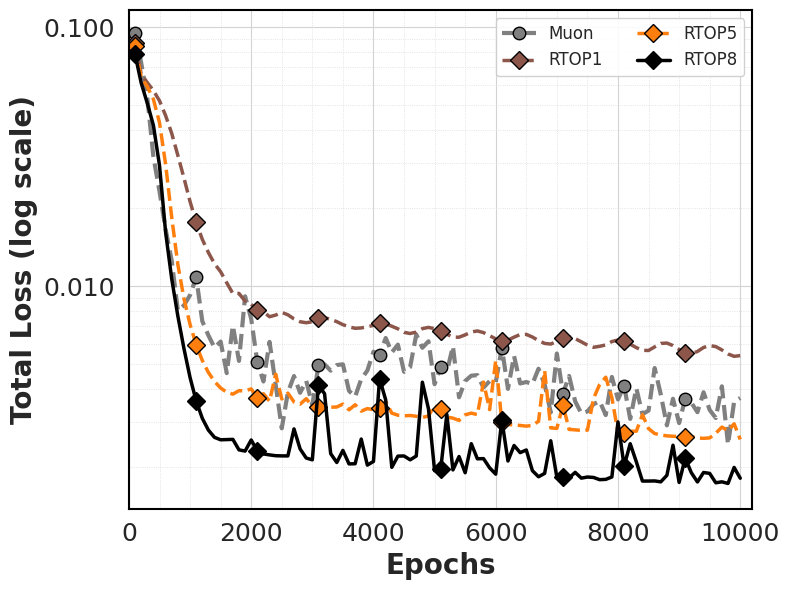}
        \caption{Ablation study on the spectral rank parameter $r_{top}$. Increasing the number of retained singular modes improves convergence up to an optimal range.}
        \label{fig:muon_sav_rtop_burger2}
    \end{subfigure}
    \caption{DeepONet training on the Burgers’ equation.}
    \label{fig:DeepONet_muon_sav_combined}
\end{figure}

The quantitative comparison is reported in Table~\ref{tab:DeepONet_results}. The SpecMuon attains the lowest final training loss and mean squared error (MSE) among all compared optimizers. Although the additional spectral computations introduce moderate computational overhead, the improved accuracy and stability demonstrate the effectiveness of integrating mode-wise energy regulation into Muon’s orthogonalized update framework.

\begin{table}[ht] \centering \caption{Quantitative comparison for the Burgers’ DeepONet experiment. Metrics include final training loss, MSE, and training time on an NVIDIA H100 80GB GPU. Best results are highlighted in bold.} 
\label{tab:DeepONet_results} 
\sisetup{separate-uncertainty, detect-weight=true, detect-inline-weight=math}  
\begin{tabular}{@{}lccc@{}} 
\toprule 
\textbf{Optimizer} & \textbf{Final Training Loss} & \textbf{MSE} & \textbf{Training Time (s)} \\ 
\midrule 
\texttt{Adam} & \num{1.89(03)e-3} & \num{0.0014} & \num{113(12)} \\ 
\texttt{AdamW} & \num{4.71(05)e-3} & \num{0.0013} & \num{114(10)} \\ 
\texttt{Muon} & \num{6.38(12)e-4} & \num{0.0011} & \num{204(13)} \\ 
\textbf{\texttt{SpecMuon}} & \textbf{\num{4.82(07)e-4}} & \textbf{\num{0.0010}} & \textbf{\num{259(10)}} \\ 
\bottomrule 
\end{tabular} 
\end{table}

\subsection{fPINN-DeepONet}\
We next consider the classical two-dimensional heat equation on the unit square with homogeneous Dirichlet boundaries:
\begin{align}
    \label{eq:2d-heat}
&\partial_t u(x,y,t) - \kappa \, \Delta u(x,y,t) \;=\; f(x,y,t),
\qquad (x,y)\in\Omega\equiv [0,1]^2,\; t\in[0,1],\\
&u|_{\partial\Omega}=0, 
\\
&u(x,y,0)=u_0(x,y).
\end{align}
To enable quantitative comparison, we consider the analytical solution
\begin{equation}
u(x,y,t) = \sin(\pi x)\sin(\pi y)\, \mathrm{e}^{-2\pi^2\kappa t},
\end{equation}
which corresponds to the homogeneous case with zero forcing, i.e., $f \equiv 0$ and initial condition $u_0(x,y) = \sin(\pi x)\sin(\pi y)$.

To ensure a fair comparison among optimizers, a grid search was conducted over the learning rate for each method. For the baseline optimizers (Adam and AdamW), standard values for the momentum coefficients and weight decay were considered. For Muon and the proposed SpecMuon, we additionally explored the influence of the spectral rank parameter $r_{top}$. The final hyperparameter configurations used in the comparative study are summarized in Table~\ref{tab:FPDE hyperparameters}.

\begin{table}[ht] \centering \caption{Hyperparameter configurations for the fPINN–DeepONet experiment on the two-dimensional heat equation. Learning rates were selected via grid search, and spectral rank $r_{top}$ was tuned for stability and accuracy.} \label{tab:FPDE hyperparameters} \begin{tabular}{@{}llcc@{}} \toprule \textbf{Optimizer} & \textbf{Hyperparameter} & \textbf{Values Searched} & \textbf{Best Value Used} \\ \midrule

 \multirow{3}{*}{\texttt{Adam}} & Learning Rate (\text{lr}) & \{0.01, 5e-3, 1e-3, 5e-4\} & 5e-3 \\ & Betas ($\beta_1, \beta_2$) & -- & (0.9, 0.999) \\ & Weight Decay ($\lambda$) & -- & 0 \\ \midrule %
 \multirow{3}{*}{\texttt{AdamW}} & Learning Rate (\text{lr}) & \{0.01, 5e-3, 1e-3, 5e-4\} & 5e-3\\ & Betas ($\beta_1, \beta_2$) & -- & (0.9, 0.999) \\ & Weight Decay ($\lambda$) & \{0.0, 0.01, 1e-3, 5e-4, 1e-4\} & 1e-2 \\ \midrule %
\multirow{2}{*}{\texttt{Muon}} & Learning Rate (\text{lr}) & \{0.1, 0.05, 5e-3, 1e-3\} & 1e-3 \\ & Momentum ($\beta$) & \{0.0, 0.02,  1e-3, 1e-4, 5e-4\} & 1e-4 \\ \midrule %
{\texttt{SpecMuon}} & Learning Rate (\text{lr}) & \{0.1, 0.05, 0.01, 1e-3, 1e-4\} & 1e-3 \\ & Momentum ($\beta$) & \{0.9,  0.01, 0.02, 1e-3, 1e-4, 5e-4\} & 0.9  \\ & r\textsubscript{top} & \{2 - 10\} & 8 \\ & SAV Eta ($\eta_{sav}$) & \{0.2 - 0.8\} & 0.2 \\ \bottomrule
 \end{tabular} 
 \end{table} 

The performance of each optimizer was evaluated using three metrics: the final training loss, the mean squared error (MSE) computed over the full space–time grid, and the relative $L_2$ error. The proposed SpecMuon optimizer achieves the lowest training loss and the smallest relative $L_2$ error among the compared methods, demonstrating improved accuracy in approximating the analytical solution.

The quantitative results are presented in Table~\ref{tab:FPDEresults}. Although the additional spectral computations introduce moderate computational overhead, the results indicate that integrating spectral decomposition with energy-based regulation enhances convergence robustness and solution quality in the fractional PINN–DeepONet framework.

\begin{figure}[h]
    \centering
    \begin{subfigure}[t]{0.48\textwidth}
        \centering
        \includegraphics[width=\textwidth]{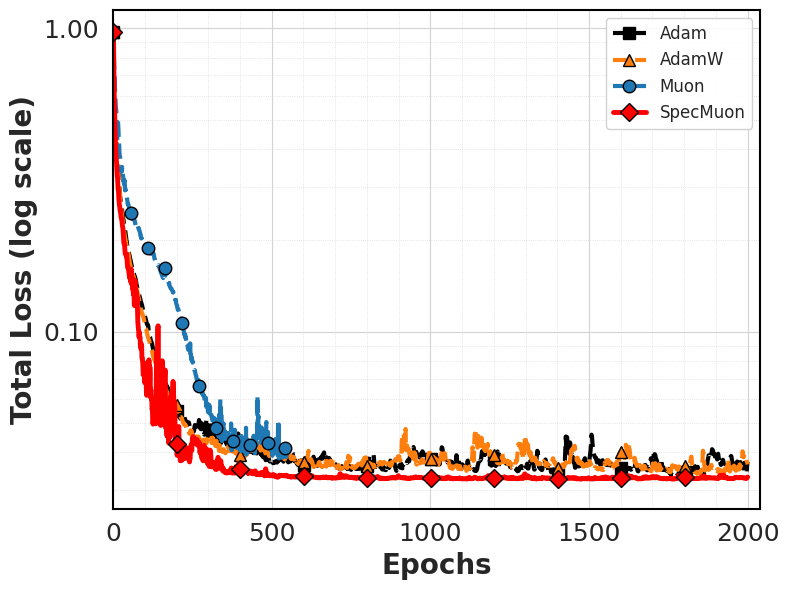}
        \caption{Training loss (log scale) comparison among Adam, AdamW, Muon, and SpecMuon.}
        \label{fig:muon_sav_fPINN-DeepONet}
    \end{subfigure}
    \hfill
    \begin{subfigure}[t]{0.48\textwidth}
        \centering
        \includegraphics[width=\textwidth]{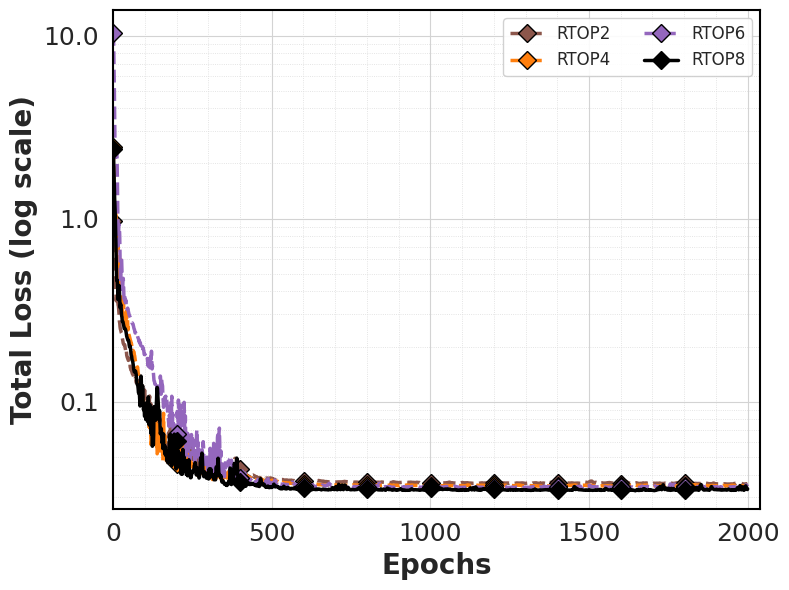}
        \caption{Ablation study on the spectral rank parameter $r_{top}$. Moderate spectral truncation provides the best balance between stability and efficiency.}
        \label{fig:muon_sav_rtop}
    \end{subfigure}
    \caption{fPINN–DeepONet training on the fractional PDE problem.}
    \label{fig:FPDE_muon_sav_combined}
\end{figure}

\begin{table}[ht] \centering \caption{Quantitative comparison for the fPINN–DeepONet experiment. Training loss is reported at 500 epochs. Training time was measured on an NVIDIA H100 80GB GPU. Best results are highlighted in bold.} 
\label{tab:FPDEresults} 
\sisetup{separate-uncertainty, detect-weight=true, detect-inline-weight=math}  
\begin{tabular}{@{}lccc@{}} 
\toprule 
\textbf{Optimizer} & \textbf{Training Loss} & \textbf{MSE} & \textbf{Training Time (s)} \\ 
\midrule 
\texttt{Adam} & \num{3.72(04)e-2} & \num{0.0268} & \num{19.41} \\ 
\texttt{AdamW} & \num{3.64(04)e-2} & \num{0.0275} & \num{13.37} \\ 
\texttt{Muon} & \num{3.52(06)e-2} & \num{0.0272} & \num{18.24} \\ 
\textbf{\texttt{SpecMuon}} & \textbf{\num{2.31(05)e-2}} & \textbf{\num{0.0262}} & \textbf{\num{30.74}} \\ 
\bottomrule 
\end{tabular} 
\end{table}

\section{Conclusion and Future Work}

In this work, we introduced SpecMuon, a spectral-aware optimization framework designed for physics-informed machine learning. By integrating orthogonalized gradient updates in the singular-vector basis with a mode-wise relaxed scalar auxiliary variable (RSAV) mechanism, SpecMuon unifies geometric conditioning and energy-based stabilization within a single optimizer. This formulation interprets optimization as a multi-mode gradient flow and enables adaptive control of stiff spectral components that commonly arise in physics-informed neural networks and neural operators.

We established rigorous theoretical properties of the proposed method, including positivity of the auxiliary variables, a modified energy dissipation law, and global convergence guarantees. Under standard smoothness assumptions and the Polyak--\L{}ojasiewicz condition, we further proved linear convergence to the global optimum. These results provide a solid mathematical foundation for the proposed algorithm and clarify the role of spectral decomposition and energy regulation in shaping stable optimization dynamics.

Extensive numerical experiments demonstrated the effectiveness of SpecMuon across a hierarchy of problems. Starting from a simple linear regression model, we showed that spectral geometry and energy-aware updates can improve optimization behavior even in convex settings. We then validated the method on increasingly challenging physics-informed learning tasks, including PINNs, DeepONets, and fractional PINN--DeepONets, where SpecMuon consistently achieved faster convergence and improved stability compared with Adam, AdamW, and the original Muon optimizer. These results confirm that incorporating spectral and energy information is particularly beneficial in stiff, multi-scale learning problems driven by physical constraints.

Several directions for future research are worth exploring. From an algorithmic perspective, extending SpecMuon to stochastic or mini-batch settings and studying its interaction with variance reduction techniques would be valuable for large-scale applications. Incorporating adaptive strategies for selecting the number of retained spectral modes may further improve efficiency and robustness. From a theoretical standpoint, sharper convergence rates under weaker assumptions and a deeper analysis of spectral evolution across iterations remain open questions. Finally, applying the proposed framework to broader classes of operator learning problems, inverse problems, and multi-physics systems may reveal additional advantages of spectral energy–aware optimization in scientific machine learning.

\section*{Acknowledgment}
We would like to thank the support of National Science Foundation (DMS-2533878, DMS-2053746, DMS-2134209, ECCS-2328241, CBET-2347401 and OAC-2311848), and U.S.~Department of Energy (DOE) Office of Science Advanced Scientific Computing Research program DE-SC0023161, the SciDAC LEADS Institute, and DOE–Fusion Energy Science, under grant number: DE-SC0024583.

\bibliographystyle{unsrt}
\bibliography{refs}

\end{document}